\documentclass{article}

\PassOptionsToPackage{numbers, compress}{natbib}

\usepackage[final]{nips_2018}

\usepackage{color}
\definecolor{mygray}{gray}{0.6}

\usepackage[utf8]{inputenc} 
\usepackage[T1]{fontenc}    
\usepackage{hyperref}       
\usepackage{url}            
\usepackage{booktabs}       
\usepackage{amsfonts}       
\usepackage{nicefrac}       
\usepackage{microtype}      
\usepackage{graphicx}

\usepackage{dsfont}
\usepackage{amsmath,amsthm}

\newtheorem{definition}{Definition}
\newtheorem{theorem}{Theorem}
\newtheorem{lemma}{Lemma}
\newtheorem{corollary}{Corollary}

\newtheorem{proposition}{Proposition}

\newcommand{\ourmethod}{RepBM}
\newcommand{\IPV}{IPV}
\newcommand{\APV}{APV}

\title{Representation Balancing MDPs \\ for Off-Policy Policy Evaluation}

\author{
  Yao Liu \\
  Stanford University\\
  \texttt{yaoliu@stanford.edu} \\
   \And
   Omer Gottesman \\
   Harvard University \\
   \texttt{gottesman@fas.harvard.edu} \\
   \AND
   Aniruddh Raghu \\
   Cambridge University \\
   \texttt{aniruddhraghu@gmail.com} \\
   \And
   Matthieu Komorowski \\
   Imperial College London \\
   \texttt{matthieu.komorowski@gmail.com } \\
   \And
   Aldo Faisal \\
   Imperial College London \\
   \texttt{a.faisal@imperial.ac.uk} \\
   \And
   Finale Doshi-Velez \\
   Harvard University \\
   \texttt{finale@seas.harvard.edu} \\
   \And
   Emma Brunskill \\
   Stanford University \\
   \texttt{ebrun@cs.stanford.edu} \\
}

\begin{document}

\maketitle

\begin{abstract}
We study the problem of off-policy policy evaluation (OPPE) in RL. In contrast to prior work, we consider how to estimate both the individual policy value and average policy value accurately. We draw inspiration from recent work in causal reasoning, and propose a new finite sample generalization error bound for value estimates from MDP models. Using this upper bound as an objective, we develop a learning algorithm of an MDP model with a balanced representation, and show that our approach can yield substantially lower MSE in common synthetic benchmarks and a HIV treatment simulation domain.
\end{abstract}

\section{Introduction}
In reinforcement learning, off-policy (batch) policy evaluation is the task of estimating the performance of some \emph{evaluation} policy given 
data gathered under a different \emph{behavior} policy.  Off-policy policy evaluation (OPPE) is essential when deploying a new policy might be costly or risky, such as in consumer marketing, healthcare, and education.  
Technically off-policy evaluation relates to other fields that study counterfactual reasoning, including causal reasoning, statistics and economics. 

Off-policy batch policy evaluation is challenging because the distribution of the data under the behavior policy will in general be different than the distribution under the desired evaluation policy. This difference in distributions comes from two sources. First, at a given state, the behavior policy may select a different action than the one preferred by the evaluation policy---for example, a clinician may chose to amputate a limb, 
whereas we may be interested in what might have happened if the clinician had not. We never see the counterfactual outcome. Second, the distribution of future states---not just the immediate outcomes---is also determined by the behavior policy. This challenge is unique to sequential decision processes and is not covered by most causal reasoning work: for example, the resulting series of a patient's health states observed after amputating a patient's limb is likely to be  significantly different than if the limb was not amputated. 

Approaches for OPPE must make a choice about whether and how to address this data distribution mismatch. Importance sampling (IS) based approaches \cite{precup2000eligibility, thomas2015high,guo2017using, dudik2011doubly,jiang2015doubly,thomas2016data} are typically unbiased and strongly consistent, but despite recent progress tend to have high variance---especially if the evaluation policy is deterministic, as evaluating deterministic policies requires finding in the data sequences where the actions exactly match the evaluation policy. However, in most real-world applications deterministic evaluation policies are more common---policies are typically to either amputate or not, rather than a policy that that flips a biased coin (to sample randomness) to decide whether to amputate. 
IS approaches also often rely on explicit knowledge of the behavior policy, which may not be feasible in situations such as medicine where the behaviors results from human actions. In contrast, some model based approaches ignore the data distribution mismatch, such as by fitting a  maximum-likelihood model of the rewards and dynamics from the behavioral data, and then using  that model to evaluate the desired evaluation policy. These methods may not converge to the true estimate of the evaluation policy's value, even in the limit of infinite data \cite{mandel2014offline}. However, such model based approaches often achieve better empirical performance than the IS-based estimators \cite{jiang2015doubly}.

In this work, we address the question of building model-based estimators for OPPE that both \emph{do} have theoretical guarantees and yield better empirical performance that model-based approaches that ignore the data distribution mismatch. Typically we evaluate the quality of an  OPPE estimate $\widehat{V}^{\pi_e}(s_0)$, where $s_0$ is an initial state, by evaluating its mean squared error (MSE).
Most previous research (e.g.~\cite{jiang2015doubly,thomas2016data}) evaluates their methods using 
MSE for the average policy value (\APV): $[ \mathbb{E}_{s_0}  \widehat{V}^{\pi_e}(s_0) - \mathbb{E}_{s_0}V^{\pi_e}(s_0) ]^2 $, rather than the MSE for individual policy values (\IPV): $  \mathbb{E}_{s_0} [ \widehat{V}^{\pi_e}(s_0) - V^{\pi_e}(s_0)]^2$.
This difference is crucial for applications such as personalized healthcare since ultimately we may want to assess the performance of a policy for an specific individual (patient) state. 

Instead, in this paper we develop an upper bound of the MSE for individual policy value estimates. Note that this bound is automatically an upper bound on the average treatment effect. Our work is inspired  by recent advances\cite{shalit2016estimating,johansson2016learning,johansson2018learning} in  estimating conditional averaged treatment effects (CATE), also known as heterogeneous treatment effects (HTE), in the contextual bandit setting with a single (typically binary) action choice.  CATE research aims to obtain precise estimates in the difference in outcomes for giving the treatment vs control intervention for an individual (state). 

Recent work \cite{johansson2016learning, shalit2016estimating} on CATE\footnote{Shalit et al.~\cite{shalit2016estimating} use the term individual treatment effect (ITE) to refer to a criterion which is actually defined as CATE in most causal inference literature. We discuss the confusion about the two terms in the appendix \ref{appendix:cate}.} has obtained very promising results by learning a model to predict individual outcomes using a (model fitting) loss function that explicitly accounts for the data distribution shift between the treatment and control policies.  
We build on this work to introduce a new bound on the MSE for individual policy values, and a new loss function for fitting a model-based OPPE estimator.
In contrast to most other OPPE theoretical analyses (e.g.~\cite{jiang2015doubly,dudik2011doubly,thomas2016data}), we provide a 
finite sample generalization error instead of asymptotic consistency. In contrast to previous model value generalization bounds such as the  Simulation Lemma \cite{kearns2002near}, our bound accounts for the underlying data distribution shift if the data used to estimate the value of an evaluation policy were collected by following an alternate policy. 

We use this to derive a loss function that we can use to fit a model for OPPE for deterministic evaluation policies. Conceptually, this process gives us a model that prioritizes fitting the trajectories in the batch data that match the evaluation policy. Our current estimation procedure works for deterministic evaluation policies 
which covers a wide range of scenarios in real-world applications that are particularly hard for previous methods. Like recently proposed IS-based estimators \cite{thomas2016data,jiang2015doubly,farajtabar2018more}, and unlike the MLE model-based estimator that ignores the distribution shift \cite{mandel2014offline}, we prove that our model-based estimator is asymptotically consistent, as long as the true MDP model is realizable within our chosen model class; we use neural models to give our model class high expressivity.



We demonstrate that our resulting models can yield substantially lower mean squared error estimators than prior model-based and IS-based estimators on a classic benchmark RL task (even when the IS-based estimators are given access to the true behavior policy). We also demonstrate our approach can yield improved results on a HIV treatment simulator \citep{ernst2006clinical}.





\section{Related Work}
Most prior work on OPPE in reinforcement learning falls into one of three approaches. The first, 
importance sampling (IS), reweights the trajectories to account for the data distribution shift. Under mild assumptions importance sampling estimators are guaranteed to be both unbiased and strongly consistent, and were first introduced to reinforcement learning OPPE by Precup et al.~\cite{precup2000eligibility}.  
Despite recent progress (e.g.\cite{thomas2015high,guo2017using})  
IS-only estimators still often yield very high variance estimates, particularly when the decision horizon is large, and/or when the evaluation policy is deterministic. IS estimators also typically result in extremely noisy estimates for policy values of individual states. A second common approach is to estimate a dynamics and reward model, which can substantially reduce variance, but can be biased and inconsistent (as noted by~\cite{mandel2014offline}). The third approach, doubly robust estimators, originates from the statistics community \cite{robins1994estimation}. Recently proposed doubly robust estimators for OPPE from the machine and reinforcement learning communities 
\cite{dudik2011doubly,jiang2015doubly,thomas2016data} have sometimes yielded orders of magnitude tighter estimates. 
However, most prior work that leverages an approximate model has largely ignored the choice of how to select and fit the model parameters. 
Recently, Farajtabar et al.~\cite{farajtabar2018more} introduced more robust doubly robust (MRDR), which involves fitting a Q function for the model-value function part of the doubly robust estimator based on 
fitting a weighted return to minimize the variance of doubly robust. In contrast, our work learns a dynamics and reward model using a novel loss function, to estimate a model that yields accurate individual policy value estimates. While our method can be combined in doubly robust estimators, we will also see in our experimental results that directly estimating the performance of the model estimator can yield substantially benefits over estimating a Q function for use in doubly robust. 


OPPE in contextual bandits and RL also has strong similarities with the treatment effect estimation problem common in causal inference and statistics.  Recently, different kinds of machine learning models such as Gaussian Processes \cite{alaa2017bayesian}, random forests \cite{wager2017estimation}, and GANs \cite{Yoon2018GANITE} have been used to estimate heterogeneous treatment effects (HTE), in non-sequential settings. Schulam and Saria \cite{schulam2017reliable} study using Gaussian process models for treatment effect estimation in continuous time settings. Their setting differs from MDPs by not having sequential states. Most theoretical analysis of treatment effects focuses on asymptotic consistency rather than generalization error. 

Our work is inspired by recent research that learns complicated outcome models (reward models in RL)  to estimate HTE using new loss functions to account for covariate shift 
\cite{johansson2016learning,shalit2016estimating,atan2018learning,johansson2018learning}. In contrast to this prior work we consider the sequential state-action setting.
In particular, Shalit et al.~\cite{shalit2016estimating} provided an algorithm with a more general model class, and a corresponding generalization bound. We extend this idea from the binary treatment setting to sequential and multiple action settings. 

\section{Preliminaries: Notation and Setting}
We consider undiscounted finite horizon MDPs, with finite horizon $H < \infty$, bounded state space $\mathcal{S} \subset \mathbb{R}^d$, and finite action space $\mathcal{A}$. Let $p_0(s)$ be the initial state distribution, and $T(s'|s,a)$ be the transition probability. Given a state action pair, the expectation of reward $r$ is $\mathbb{E}[r|x,a] = \bar{r}(x,a)$. 
Given $n$ trajectories collected from a stochastic behavior policy $\mu$, our goal is to evaluate the policy value of $\pi(s)$. We assume the policy $\pi(s)$ is deterministic. We will learn a model of both reward and transition dynamics, $\widehat{M}= \langle \widehat{r}(s,a), \widehat{T}(s',s,a) \rangle$, based on a learned representation. The representation function $\phi: \mathcal{S} \mapsto \mathcal{Z}$ is a reversible and twice-differentiable function, where $\mathcal{Z}$ is the representation space. $\psi$ is the reverse representation such that $\psi(\phi(s))=s$. The specific form of our MDP model is: $\widehat{M}_{\phi} =\langle \widehat{r}(s,a), \widehat{T}(s',s,a) \rangle  = \langle h_{r}(\phi(s),a), h_{T}(\phi(s'),\phi(s),a) \rangle$, where $h_{r}$ and $h_{T}$ is some function over space $\mathcal{Z}$. We will use the notation $\widehat{M}$ instead of $\widehat{M}_\phi$ later for simplicity.

Let $\tau = (s_0, a_0, \dots, s_H)$ be a trajectory of $H+1$ states and actions, sampled from the joint distribution of MDP $M$ and a policy $\mu$. The joint distributions of $\tau$ are:
$p_{M,\mu}(\tau) = p_0(s_0)\prod_{t=0}^{H-1} \left[ T(s_{t+1}|s_t,a_t) \mu(a_t|s_t) \right]$.
Given the joint distribution, we denote the associated marginal and conditional distributions  as $p_{M,\mu}(s_0), p_{M,\mu}(s_0, a_0), p_{M,\mu}(s_0|a_0)$ etc. 
We also have the joint, marginal and conditional, distributions $p_{M,\mu}^{\phi}(\cdot)$ based on the representation space $\mathcal{Z}$. 
We focus on the undiscounted finite horizon case, using  $V_{M,t}^\pi(s)$ to denote the $t$-step value function of policy  $\pi$.

\section{Generalization Error Bound for MDP based OPPE estimator}
\label{sec:theory}
Our goal is to learn a MDP model $\widehat{M}$ that directly minimizes a good upper bound of the MSE for the individual evaluation policy $\pi$ values: $\mathbb{E}_{s_0} [ V^\pi_{\widehat{M}}(s_0) - V^\pi_{M}(s_0) ]^2$. This model can provide value function estimates of the policy $\pi$ and be used as part of doubly robust methods. 

In the on-policy case, the Simulation Lemma (~\cite{kearns2002near} and repeated for completeness in Lemma \ref{lem:simulation}) shows that MSE of a policy value estimate can be upper bounded by a function of the reward and transition prediction losses. Before we state this result, we first define some useful notation.
\begin{definition}
\label{def:lossoverexp}
The square error loss function of value function, reward, transition are:
\begin{gather}
\bar{\ell}_{V}(s,\widehat{M},H-t) = \left(V_{\widehat{M},H-t}^\pi(s) - V_{M,H-t}^\pi(s) \right)^2  \quad
\bar{\ell}_{r}(s_t,a_t,\widehat{M}) = \left( \widehat{r}(s_t,a_t) - \bar{r}(s_t,a_t) \right)^2 \nonumber \\
\bar{\ell}_{T}(s_t,a_t,\widehat{M}) = \left( \int_{\mathcal{S}} \left( \widehat{T}(s'|s_t,a_t) -T(s'|s_t,a_t) \right)V_{\widehat{M},H-t-1}^{\pi}(s') ds' \right)^2 
\label{eqn:def_loss}
\end{gather}
\end{definition}
Then the Simulation lemma ensures that 
\begin{equation}
\label{eqn:onpolicyloss}
\mathbb{E}_{s_0 } \left[ V^\pi_{\widehat{M}}(s_0) - V^\pi_{M}(s_0) \right]^2 \le 2H\sum_{t=0}^{H-1}\mathbb{E}_{s_t, a_t \sim p_{M,\pi}} \left[ \bar{l}_r(s_t,a_t,\hat{M})+\bar{l}_T(s_t,a_t,\hat{M}) \right],   
\end{equation}
The right hand side can be used to formulate an objective to fit a model for policy evaluation. In off-policy case our data is from a different policy $\mu$, and one can get unbiased estimation of the RHS of Equation \ref{eqn:onpolicyloss} by importance sampling.
However, this will provide an objective function with high variance, especially for a long horizon MDP or a deterministic evaluation policy due to the product of IS weights.
An alternative is to learn an MDP model by directly optimizing the prediction loss over our observational data, ignoring the covariate shift. From the Simulation Lemma this minimizes an upper bound of MSE of behavior policy value, but the resulting model may not be a good one for estimating the evaluation policy value.
In this paper we propose a new upper bound on the MSE of the individual evaluation policy values inspired by recent work in treatment effect estimation, and use this as a loss function for fitting models.


Before proceeding we first state our assumptions, which are common in most OPPE algorithms: 
\begin{enumerate}
    \item Support of behavior policy covers the evaluation policy: for any state $s$ and action $a$, $\mu(a|s) = 0$ only if $\pi(a|s) = 0$. 
    \item Strong ignorability: there are no hidden confounders that influence the choice of actions other than the current observed state.  
\end{enumerate}


Denote a \textit{factual} sequence to be a trajectory that matches the evaluation policy, $a_0 = \pi(s_{0}), \dots, a_{t-1} = \pi(s_{t-1})$ as $a_{0:t-1} = \pi$. Let a \textit{counterfactual} action sequence $a_{0:t-1} \neq \pi$ be an action sequence with at least one action that does not match $\pi(s)$. $p_{M,\mu}(\cdot)$ is the distribution over trajectories under $M$ and policy $\mu$. We define the $H-t$ step value error with respect to the state distribution given the factual action sequence.
\begin{definition}
\label{def:valueerror}
$H-t$ step value error is:
$ \epsilon_{V}(\widehat{M},H-t) = \int_{\mathcal{S}} \bar{\ell}_{V}(s_t,H-t) p_{M,\mu}(s_t|a_{0:t-1} = \pi) ds_t $
\end{definition}
We use the idea of bounding the distance between representations given factual and counterfactual action sequences to adjust the distribution mismatch. Here the distance between representation distributions is formalized by Integral Probability Metric (IPM). 
\begin{definition}
\label{def:IPM}
Let $p,q$ be two distributions and let  $G$ be a family of real-valued functions defined over the same space.  The integral probability metric is: $\text{IPM}_{G}(p,q) = \sup_{g \in G} \left| \int g(x)(p(x)-q(x))dx \right| $
\end{definition}
Some important instances of IPM include the Wasserstein metric where $G$ is 1-Lipschitz continuous function class, and Maximum Mean Discrepancy where $G$ is norm-1 function class in RKHS.

Let $p_{M,\mu}^{\phi, F}(z_t) = p_{M,\mu}^{\phi}(z_t|a_{0:t}=\pi)$ and $p_{M,\mu}^{\phi, CF}(z_t) = p_{M,\mu}^{\phi}(z_t|a_t \neq \pi, a_{0:t-1} = \pi)$, where $F$ and $CF$ denote factual and counterfactual. We first give an upper bound of MSE in terms of an expected loss term and then develop a finite sample bound which can be used as a learning objective. 
\begin{theorem}
\label{thm:MSEVpi}
For any MDP $M$, approximate MDP model $\widehat{M}$, behavior policy $\mu$ and deterministic evaluation policy $\pi$, let $B_{\phi,t}$ and $G_t$ be a real number and function family that satisfy the condition in Lemma \ref{lem:MSEVpirecursive}. Then:
\begin{small}
\begin{multline}
\mathbb{E}_{s_0}\left[ V^\pi_{\widehat{M}}(s_0) - V^\pi_{M}(s_0) \right]^2 \le 2H \sum_{t=0}^{H-1} \left[ B_{\phi,t} \text{IPM}_{G_{t}}\left(p_{M,\mu}^{\phi, F}(z_t),p_{M,\mu}^{\phi, CF}(z_t)\right) \right. \\
\left. + \int_{\mathcal{S}}\frac{1}{p_{M,\mu}(a_{0:t}=\pi)}\left( \bar{\ell}_{r}(s_t,\pi(s_t),\widehat{M})+\bar{\ell}_{T}(s_t,\pi(s_t),\widehat{M}) \right) p_{M,\mu}(s_t,a_{0:t}=\pi) d s_t  \right]
\label{eqn:thm1loss}
\end{multline}
\end{small}
\end{theorem}

\textbf{(Proof Sketch)}
The key idea is to use Equation \ref{eqn:simulationlemmarecursive} in Lemma \ref{lem:simulation} to view each step as a contextual bandit problem, and bound  $\epsilon_{V}(\widehat{M},H)$ recursively. 
We decompose the value function error into a one step reward loss, a transition loss and a next step value loss, with respect to the on-policy distribution. We can treat this as a contextual bandit problem, and we build on the method in Shalit et al.'s work \cite{shalit2016estimating} about binary action bandits to bound the distribution mismatch by a representation distance penalty term; however, additional care is required due to the sequential setting since the next states are also influenced by the policy. 
By adjusting the distribution for the next step value loss, we reduce it into $\epsilon_V(\widehat{M},H-t-1)$, allowing us recursively repeat this process for H steps. 
\qed

This theorem bounds the MSE for the individual evaluation policy value by a loss on the distribution of the behavior policy, with the cost of an additional representation distribution metric. The first IPM term measures how different the state representations are conditional on factual and counterfactual action history. Intuitively, a balanced representation can generalize better from the observational data distribution to the data distribution under the evaluation policy, but we also need to consider the prediction ability of the representation on the observational data distribution. This bound quantitatively describes those two effects about MSE by the IPM term and the loss terms. The re-weighted expected loss terms over the \emph{observational} data distribution is weighted by the marginal action probabilities ratio instead of the conditional action probability ratio, which is used in importance sampling. The marginal probabilities ratio has lower variance than the importance sampling weights (See Appendix \ref{appendix:theory_variance}). 

One natural approach might be to use the right hand side of Equation \ref{eqn:thm1loss} as a loss, and try to directly optimize a representation and model that minimizes this upper bound on the mean squared error in the individual value estimates. Unfortunately, doing so can suffer from two important issues. (1) The subset of the data that matches the evaluation policy can be very sparse for large $t$, and though the above bound re-weights data, fitting a model to it can be challenging due to the limited data size. (2) Unfortunately this approach ignores all the other data present that do not match the evaluation policy. If we are also learning a representation of the domain in order to scale up to very large problems, we suspect that we may benefit from framing the problem as related to transfer or multitask learning. 

Motivated by viewing off-policy policy evaluation as a transfer learning task, we can view the source task as the evaluating the behavior policy, for which we have on-policy data, and view the target task as evaluating the evaluation policy, for which we have the high-variance re-weighted data from importance sampling. 
This is similar to transfer learning where we only have a few, potentially noisy, data points for the target task. Thus we can take the idea of co-learning a source task and a target task at the same time as a sort of regularization given limited data. More precisely, we now bound the OPPE error by an upper bound of the sum of two terms:
\begin{small}
\begin{equation}
\label{eqn:totalobj}
\underbrace{\mathbb{E}_{s_0 } \left[ V^\pi_{\widehat{M}}(s_0) - V^\pi_{M}(s_0) \right]^2}_{\text{MSE}_\pi} + \underbrace{\mathbb{E}_{s_0 } \left[ V^\mu_{\widehat{M}}(s_0) - V^\mu_{M}(s_0) \right]^2}_{\text{MSE}_\mu},
\end{equation}
\end{small}
where we bound the former part using Theorem \ref{thm:MSEVpi}. Thus our upper bound of this objective can address the issues with separately using $\text{MSE}_{\pi}$ and $\text{MSE}_{\mu}$ as objective: compared with IS estimation of $\text{MSE}_{\pi}$, the "marginal" action probability ratio has lower variance. The representation distribution distance term regularizes the representation layer such that the learned representation would not vary significantly between the state distribution under the evaluation policy and the state distribution under the behavior policy. That reduces the concern that using $\text{MSE}_{\mu}$ as an objective will force our model to evaluate the behavior policy, rather than the evaluation policy, more effectively.

Our work is also inspired by treatment effect estimation in the casual inference literature, where we estimate the difference between the treated and control groups. An analogue in RL would be estimating the difference between the target policy value and the behavior policy value, by minimizing the MSE of policy difference estimation. The objective above is an upper bound of the MSE of policy difference estimator:
$
\frac{1}{2}\mathbb{E}_{s_0 } \left[ \left(V^\pi_{\widehat{M}}(s_0) -  V^\mu_{\widehat{M}}(s_0) \right) - \left( V^\pi_{M}(s_0) - V^\mu_{M}(s_0)  \right) \right]^2 
\le \text{MSE}_{\pi} + \text{MSE}_{\mu}
$

We now bound Equation \ref{eqn:totalobj} further by finite sample terms. For the finite sample generalization bound, we first introduce a minor variant of the loss functions, with respect to the sample set.
\begin{definition}
\label{def:lossoversample}
Let $r_t$ and $s'_t$ be an observation of reward and next step given state action pair $s_t, a_t$. Define the loss functions as:
\begin{align}
\ell_{r}(s_t,a_t,r_t,\widehat{M}) =& \left( \widehat{r}(s_t,a_t) - r_t \right)^2 \label{eqn:def_rewardlossonsample}\\
\ell_{T}(s_t,a_t,s'_t,\widehat{M}) =& \left( \int_{\mathcal{S}} \widehat{T}(s'|s_t,a_t) V_{\widehat{M},H-t-1}^{\pi}(s') ds' - V_{\widehat{M},H-t-1}^{\pi}(s'_t)  \right)^2
\label{eqn:def_transitionlossonsample}
\end{align}
\end{definition}

\begin{definition}
\label{def:emprisk}
Define the empirical risk over the behavior distribution and weighted distribution as:
\begin{small}
\begin{align}
\widehat{R}_{\mu}(\widehat{M}) =& \frac{1}{n} \sum_{i=1}^n \sum_{t=0}^{H-1} \ell_r(s_t^{(i)},a_t^{(i)},r^{(i)},\widehat{M}) + \ell_T(s_t^{(i)},a_t^{(i)},s'^{(i)}_t,\widehat{M}) \\
\widehat{R}_{\pi,u}(\widehat{M}) =& \frac{1}{n} \sum_{i=1}^n \sum_{t=0}^{H-1} \frac{\mathds{1}(a^{(i)}_{0:t}=\pi)}{\widehat{u}_{0:t}}\left[ \ell_r(s_t^{(i)},a_t^{(i)},r^{(i)},\widehat{M}) + \ell_T(s_t^{(i)},a_t^{(i)},s'^{(i)}_t,\widehat{M}) \right],
\end{align}
\end{small}
where n is the dataset size, $s^{(i)}_t$ is the state of the $t^{\text{th}}$ step in the $i^{\text{th}}$ trajectory, and $\widehat{u}_{0:t} = \sum_{i=1}^n \frac{\mathds{1}(a^{(i)}_{0:t}=\pi)}{n}$.
\end{definition}

\begin{theorem}
\label{thm:finite}
Suppose $\mathcal{M}_{\Phi}$ is a model class of MDP models based on representation $\phi$. For n trajectories sampled by $\mu$, let $\ell_t(s_t,a_t, \widehat{M}_{\phi}) = \ell_{r}(s_t,a_t,r_t,\widehat{M}) + \ell_{T}(s_t,a_t,s'_t,\widehat{M})$, and $d_t$ be the pseudo-dimension of function class $\{ \ell_t(s_t,a_t, \widehat{M}_{\phi}), \widehat{M}_{\phi} \in \mathcal{M}_{\Phi}\}$. Suppose $\mathcal{H}$ is the reproducing kernel Hilbert space induced by k, and $\mathcal{F}$ is the unit ball in it. Assume there exists a constant $B_{\phi,t}$ such that $\frac{1}{B_{\phi,t}}\ell_t(\psi(z),\pi(\psi(z)), \widehat{M}_{\phi}) \in \mathcal{F}$. With probability $1-3\delta$, for any $\widehat{M} \in \mathcal{M}_{\Phi}$:
\begin{multline}
\label{eqn:thmfinite}
\mathbb{E}_{s_0} \left[ V^\pi_{\widehat{M}}(s_0) - V^\pi_{M}(s_0) | \widehat{M}  \right]^2 \le \text{MSE}_{\pi} + \text{MSE}_{\mu} \le 2H\widehat{R}_{\mu}(\widehat{M}) + 2H\widehat{R}_{\pi,u}(\widehat{M}) \\ 
+ 2H\sum_{t=0}^{H-1} B_{\phi,t} \left( \text{IPM}_{\mathcal{F}}\left(\widehat{p}_{M,\mu}^{\phi,F}(z_t),\widehat{p}_{M,\mu}^{\phi, CF}(z_t) \right) + \min \left\{ \mathcal{D}^{\mathcal{F}}_{\delta} \left( \frac{1}{\sqrt{m_{t,1}}}+\frac{1}{\sqrt{m_{t,2}}} \right), 2\nu \right\} \right) \\
+ 2H\sum_{t=0}^{H-1}\frac{\mathcal{C}^{\mathcal{M}}_{n,\delta, t}}{n^{3/8}} \left( \mathbb{V}[\frac{\mathds{1}(a_{0:t}=\pi)}{\widehat{u}_{0:t}},\ell_t] + \mathbb{V}[1,\ell_t] +\ell_{t,\max}\mathbb{V}[\frac{\mathds{1}(a_{0:t}=\pi)}{u_{0:t}},1] \right)
\end{multline}
$m_{t,1}$ and $m_{t,2}$ are the number of samples used to estimate $\widehat{p}_{M,\mu}^{\phi,F}(z_t)$ and $\widehat{p}_{M,\mu}^{\phi,CF}(z_t)$ respectively. $\mathcal{D}^{\mathcal{F}}_{\delta}$ is a function of the kernel $k$. $\mathcal{C}^{\mathcal{M}}_{n,\delta, t}$ is a function of $d_t$. $\mathbb{V}[w,\ell_t] = \max \{ \sqrt{\mathbb{E}_{p_{M,\mu}}[w^2\ell_t^2] }, \sqrt{\mathbb{E}_{\widehat{p}_{M,\mu}}[w^2\ell_t^2] } \}$. $\ell_{t,\max} = \max_{s_t,a_t} |\ell_t(s_t,a_t)|$.
\end{theorem}

The first term is the empirical loss over the observational data distribution. The second term is a re-weighted empirical loss, which is an empirical version of the first term in Theorem \ref{thm:MSEVpi}. As said previously, this re-weighting has less variance than importance sampling in practice, especially when the sample size is limited. Theorem \ref{thm:variance} in Appendix \ref{appendix:theory_variance} shows that the variance of this ratio is also no greater than the variance of IS weights. Our bound is  based on the empirical estimate of the marginal probability $u_{0:t}$ and we are not required to know the behavior policy. Our method's independence of the behavior policy is a significant advantage over IS methods which are very susceptible to errors its estimation, as we discuss in appendix \ref{appendix:empiricalbehavior}. In practice, this marginal probability $u_{0:t}$ is easier to estimate than $\mu$ when $\mu$ is unknown. The third term is an empirical estimate of IPM, which we described in Theorem \ref{thm:MSEVpi}. We use norm-1 RKHS functions and MMD distance in this theorem and our algorithm. There are similar but worse results for Wasserstein distance and total variation distance \cite{sriperumbudur2009integral}. $\mathcal{D}^{\mathcal{F}}_{\delta}$ measures how complex $\mathcal{F}$ is. It is obtained from concentration measures about empirical IPM estimators \cite{sriperumbudur2009integral}. The constant $\mathcal{C}^{\mathcal{M}}_{n,\delta, t}$ measures how complex the model class is and it is derived from traditional learning theory results \cite{cortes2010learning}.

We compare our bound with the upper bound of  model error for OPPE in \cite{hanna2017bootstrapping}. In the corrected version of corollary 2 in \cite{hanna2017bootstrapping}, the upper bound of absolute error has a linear dependency on $\sqrt{\bar{\rho}_{1:H}}$ where $\bar{\rho}_{1:H}$ is an upper bound of the importance ratio, which is usually a dominant term in long horizon cases. As we clarified in last paragraph, the re-weighting weights in our bound, which are marginal action probability ratios, enjoy a lower variance than IS weights (See Appendix \ref{appendix:theory_variance}).

\section{Algorithm for Representation Balancing MDPs}
\label{sec:algorithm}
Based on our generalization bound above, we propose an algorithm to learn an MDP model for OPPE, minimizing the following objective function:
\begin{equation}
\mathcal{L}(\widehat{M}_{\phi}; \alpha_t) =  \widehat{R}_{\mu}(\widehat{M}_{\phi}) + \widehat{R}_{\pi,u}(\widehat{M}_{\phi}) + \sum\nolimits_{t=0}^{H-1}\alpha_t \text{IPM}_{\mathcal{F}}\left(\widehat{p}_{M,\mu}^{\phi,F}(z_t),\widehat{p}_{M,\mu}^{\phi,CF}(z_t) \right) 
 + \frac{\mathfrak{R}(\widehat{M}_{\phi})}{n^{3/8}}
\end{equation} 
This objective is based on Equation \ref{eqn:thmfinite} in Theorem \ref{thm:finite}. We minimize the terms in this upper bound that are related to the model $\widehat{M}_{\phi}$. Note that since $B_{\phi,t}$ depends on the loss function, we cannot know $B_{\phi,t}$ in practice. We therefore use a tunable factor $\alpha$ in our algorithm. 
$\mathfrak{R}(\widehat{M}_{\phi})$ here is some kind of bounded regularization term of model that one could choose, corresponding to the model class complexity term in Equation \ref{eqn:thmfinite}. This objective function matches our intuition about using lower-variance weights for the re-weighting component and using IPM of the representation to avoid fitting the behavior data distribution.

In this work, $\phi(s)$ and $\widehat{M}_{\phi}$ are parameterized by neural networks, due to their strong ability to learn representations. We use an estimator of IPM term from Sriperumbudur et al. \cite{sriperumbudur2012empirical}. All terms in the objective function are differentiable, allowing us to train them jointly by minimizing the objective by a gradient based optimization algorithm.

After we learn an MDP by minimizing the objective above, we use Monte-Carlo estimates or value iteration to get the value for any initial state $s_0$ as an estimator of policy value for that state. We show that if there exists an MDP and representation model in our model class that could achieve:
\[
\min_{\widehat{M}_{\phi}} \left(R_{\mu}(\widehat{M}_{\phi}) + R_{\pi,u}(\widehat{M}_{\phi}) + \sum\nolimits_{t=0}^{H-1}\alpha_t \text{IPM}_{\mathcal{F}}\left(p_{M,\mu}^{\phi,F}(z_t),p_{M,\mu}^{\phi,CF}(z_t) \right) \right) = 0,
\]
then $\lim_{n\to \infty}\mathbb{E}_{s_0} [ V^\pi_{\widehat{M}^*_{\phi^*}}(s_0) - V^\pi_{M}(s_0) ]^2 \to 0$ and estimator $V^\pi_{\widehat{M}^*_{\phi^*}}(s_0)$ is a consistent estimator for any $s_0$. See Corollary \ref{cor:consistency} in Appendix for detail.

We can use our model in any OPPE estimators that leverage model-based estimators, such as doubly robust \cite{jiang2015doubly} and MAGIC \cite{thomas2016data}, though our generalization MSE bound is just for the model value.

\section{Experiments}
\subsection{Synthetic control domain: Cart Pole and Montain Car}
We test our algorithm on two continuous-state benchmark domains. We use a greedy policy from a learned Q function as the evaluation policy, and an $\epsilon$-greedy policy with $\epsilon=0.2$ as the behavior policy. We collect 1024 trajectories for OPPE. In Cart Pole domain the average length of trajectories is around 190 (long horizon variant), or around 23 (short horizon variant). In Mountain Car the average length of trajectories is around 150. The long horizon setting (H>100) is challenging for IS-based OPPE estimators due to the deterministic evaluation policy and long horizon, which will give the IS weights high variance. Deterministic dynamics and long horizons are common in real-world domains, and most off policy policy evaluation algorithms struggle in such scenarios. 

We compare our method \ourmethod, with two baseline approximate models (AM and AM($\pi$)), doubly robust (DR), more robust doubly robust (MRDR), and importance sampling (IS). The baseline approximate model (AM) is an MDP model-based estimator trained by minimizing the empirical risk, using the same model class as \ourmethod. AM($\pi$) is an MDP model trained with the same objective as our method but without the $\text{MSE}_{\mu}$ term. DR is a doubly robust estimator using our model and DR(AM) is a doubly robust estimator using the baseline model. MRDR \cite{farajtabar2018more} is a recent method that trains a Q function as the model-based part in DR to minimize the resulting variance. We include their Q function estimator (MRDR Q), the doubly robust estimator that combines this Q function with IS (MRDR). 

The reported results are square root of the average MSE over $100$ runs. $\alpha$ is set to $0.01$ for \ourmethod. We report mean and individual MSEs, corresponding to MSEs of average policy value and individual policy value, $[\mathbb{E}_{s_0}\widehat{V}(s_0) - \mathbb{E}_{s_0}V(s_0)]^2$ and $\mathbb{E}_{s_0}[\widehat{V}(s_0) - V(s_0)]^2$ respectively. IS and DR methods re-weight samples, so their estimates for single initial states are not applicable, especially in continuous state space. A comparison across more methods is included in the appendix.

\begin{table}[ht!]
\vspace{-0.5em}
\caption{Root MSE for Cart Pole}
\small
\centering
\begin{tabular}{ccccccccc}
\toprule
\multicolumn{1}{c}{Long Horizon} & 
\multicolumn{1}{c}{\ourmethod} & 
\multicolumn{1}{c}{DR} & 
\multicolumn{1}{c}{AM} & 
\multicolumn{1}{c}{DR(AM)} &
\multicolumn{1}{c}{AM($\pi$)} & 
\multicolumn{1}{c}{MRDR Q} &
\multicolumn{1}{c}{MRDR} &
\multicolumn{1}{c}{IS} 
\\
\midrule
Mean &\textbf{0.4121} &1.359 &0.7535 &1.786 &41.80 &151.1 &202 &194.5 \\
Individual  &\textbf{1.033} &- &1.313 &- &47.63 &151.9 &- &- \\
\toprule
\multicolumn{1}{c}{Short Horizon} & 
\multicolumn{1}{c}{\ourmethod} & 
\multicolumn{1}{c}{DR} & 
\multicolumn{1}{c}{AM} & 
\multicolumn{1}{c}{DR(AM)} &
\multicolumn{1}{c}{AM($\pi$)} & 
\multicolumn{1}{c}{MRDR Q} &
\multicolumn{1}{c}{MRDR} &
\multicolumn{1}{c}{IS} 
\\
\midrule
Mean &0.07836 &\textbf{0.02081} &0.1254 &0.0235 &0.1233 &3.013 &0.258 &2.86 \\
Individual  &\textbf{0.4811} &- &0.5506 &- &0.5974 &3.823 &- &- \\
\bottomrule
\end{tabular}
\vspace{-0.5em}
\end{table}

\begin{table}[ht!]
\vspace{-0.5em}
\caption{Root MSE for Mountain Car}
\small
\centering
\begin{tabular}{ccccccccc}
\toprule
\multicolumn{1}{c}{} & 
\multicolumn{1}{c}{\ourmethod} & 
\multicolumn{1}{c}{DR} & 
\multicolumn{1}{c}{AM} & 
\multicolumn{1}{c}{DR(AM)} &
\multicolumn{1}{c}{AM($\pi$)} & 
\multicolumn{1}{c}{MRDR Q} &
\multicolumn{1}{c}{MRDR} &
\multicolumn{1}{c}{IS} 
\\
\midrule
Mean &\textbf{12.31} &135.8 &17.15 &141.6 &72.61 &135.4 &172.7 &149.7 \\
Individual  &\textbf{31.38} &- &36.36 &- &79.46 &138.1 &- &- \\
\bottomrule
\end{tabular}
\vspace{-0.5em}
\end{table}

\textbf{Representation Balancing MDPs outperform baselines for long time horizons.} We observe that MRDR variants and IS methods have high MSE in the long horizon setting. The reason is that the IS weights for 200-step trajectories are extremely high-variance, and MRDR whose objective depends on the square of IS weights, also fails. Compared with the baseline model, we can see that our method is better than AM for both the pure model case and when used in doubly robust. We also observe that the IS part in doubly robust actually hurts the estimates, for both \ourmethod~and AM. 

\textbf{Representation Balancing MDPs outperform baselines in deterministic settings.} To observe the benefit of our method beyond long horizon cases, we also include results on Cart Pole with a shorter horizon, by using weaker evaluation and behavior policies. The average length of trajectories is about 23 in this setting. Here, we observe that \ourmethod~is still better than other model-based estimators, and doubly robust that uses \ourmethod~is still better than other doubly robust methods. Though MRDR produces substantially lower MSE than IS, which matches the report in Farajtabar et al.~\cite{farajtabar2018more}, it still has higher MSE than \ourmethod~and AM, due to the high variance of its learning objective when the evaluation policy is deterministic.

\textbf{Representation Balancing MDPs produce accurate estimates even when the behavior policy is unknown.} For both horizon cases, we observe that \ourmethod~learned with no knowledge of the behavior policy is better than methods such as MRDR and IS that use the true behavior policy. 

\subsection{HIV simulator}
\label{sec:hiv_sepsis}
We demonstrate our method on an HIV treatment simulation domain. The simulator is described in Ernst et al. \cite{ernst2006clinical}, and consists of 6 parameters describing the state of the patient and 4 possible actions. The HIV simulator has richer dynamics than the two simple control domains above. We learn an evaluation policy by fitted Q iteration and use the $\epsilon$-greedy policy of the optimal Q function as the behavior policy.

We collect 50 trajectories from the behavior policy and learn our model with the baseline approximate model (AM). We compare the root average MSE of our model with the baseline approximate MDP model, importance sampling (IS), per-step importance sampling (PSIS) and weighted per-step importance sampling (WPSIS). The root average MSEs reported are averaged over 80 runs. We observe that \ourmethod~has the lowest root MSE on estimating the value of the evaluation policy.

\begin{table}[ht!]
\vspace{-0.5em}
\caption{Relative Root MSE for HIV}
\small
\centering
\begin{tabular}{cccccc}
\toprule
\multicolumn{1}{c}{} & 
\multicolumn{1}{c}{\ourmethod} & 
\multicolumn{1}{c}{AM} & 
\multicolumn{1}{c}{IS} &
\multicolumn{1}{c}{PSIS} &
\multicolumn{1}{c}{WPSIS} 
\\
\midrule
Mean &\textbf{0.062} &0.067 &0.95 &0.273 &0.146  \\
\bottomrule
\end{tabular}
\vspace{-0.5em}
\end{table}

\section{Discussion and Conclusion}
One interesting issue for our method is the effect of the hyper-parameter $\alpha$ on the quality of estimator. In the appendix, we include the results of \ourmethod~across different values of $\alpha$. We find that our method outperforms prior work for a large range of alphas, for both domains. 
In both domains we observe that the effect of IPM adjustment (non-zero $\alpha$) is less than the effect of "marginal" IS re-weighting, which matches the results in Shalit et al.'s work in the binary action bandit case \cite{shalit2016estimating}. 

To conclude, in this work we give an MDP model learning method for the individual OPPE problem in RL, based on a new finite sample generalization bound of MSE for the model value estimator. 
We show our method results in substantially smaller MSE estimates compared to state-of-the-art baselines in common benchmark control tasks and on a more challenging HIV simulator. 



\subsubsection*{Acknowledgments}
This work was supported in part by the Harvard Data Science Initiative, Siemens, and a NSF CAREER grant.

\begin{small}
\bibliographystyle{abbrv}
\bibliography{ref.bib}
\end{small}

\newpage
\appendix
\setcounter{theorem}{0}
\setcounter{corollary}{0}

\section{IS with approximate behavior policy}
\label{appendix:empiricalbehavior}
In this section, we include some theoretical and empirical results about the effect of using an estimated behavior policy in importance sampling, when the true behavior policy is not accessible.

\begin{proposition}
Assume the reward is in $[0,R_{\max}]$. For any estimator $\widehat{\mu}(a|s)$ of the true behavior policy $\mu(a|s)$, let $V_{IS}(\hat{\mu})$ be the IS estimator using this estimated $\hat{\mu}(a|s)$ and $V_{IS}(\mu)$ be the IS estimator with true behavior policy. Both of the IS estimators are computed using $n$ trajectories that are independent from the data used to estimate $\widehat{\mu}(a|s)$. If the relative error of $\widehat{\mu}(a|s)$ is bounded by $\delta$: $\| \frac{\widehat{\mu}(a|s)-\mu(a|s)}{\mu(a|s)} \|_{\infty} \le \delta $, then for any given dataset:
\[ | V_{IS}(\hat{\mu}) - V_{IS}(\mu)  | \le\max \left\{ \left( \frac{1}{1-\delta} \right)^{H}-1, 1-\left( \frac{1}{1+\delta} \right)^{H} \right\} V_{IS}(\mu)  \]
The bias of $V_{IS}(\hat{\mu})$ is bounded by:
\[ | \mathbb{E}V_{IS}(\hat{\mu}) - v  | \le \max \left\{ \left( \frac{1}{1-\delta} \right)^{H}-1, 1-\left( \frac{1}{1+\delta} \right)^{H} \right\}v,  \]
where $v$ is the true evaluation policy value.
\end{proposition}
\begin{proof}
\begin{eqnarray}
\left| V_{IS}(\hat{\mu}) - V_{IS}(\mu) \right|  &=& \left|\frac{1}{n}\sum_{i=1}^n \prod_{t=0}^{H-1}\frac{\pi(a^{(i)}_t|s^{(i)}_t)}{\hat{\mu}(a^{(i)}_t|s^{(i)}_t)} R^{(i)}_{0:H-1} - \frac{1}{n}\sum_{i=1}^n \prod_{t=0}^{H-1}\frac{\pi(a^{(i)}_t|s^{(i)}_t)}{\mu(a^{(i)}_t|s^{(i)}_t)} R^{(i)}_{0:H-1}\right| \nonumber \\
&\le& \frac{1}{n}\sum_{i=1}^n \left| \prod_{t=0}^{H-1}\frac{\pi(a^{(i)}_t|s^{(i)}_t)}{\hat{\mu}(a^{(i)}_t|s^{(i)}_t)} R^{(i)}_{0:H-1}- \prod_{t=0}^{H-1}\frac{\pi(a^{(i)}_t|s^{(i)}_t)}{\mu(a^{(i)}_t|s^{(i)}_t)} R^{(i)}_{0:H-1} \right| \\
&\le& \frac{1}{n}\sum_{i=1}^n \left| \prod_{t=0}^{H-1}\frac{\mu(a^{(i)}_t|s^{(i)}_t)}{\hat{\mu}(a^{(i)}_t|s^{(i)}_t)}- 1 \right| \left| \prod_{t=0}^{H-1}\frac{\pi(a^{(i)}_t|s^{(i)}_t)}{\mu(a^{(i)}_t|s^{(i)}_t)} R^{(i)}_{0:H-1} \right|
\label{eqn:empiricalpibdiff}
\end{eqnarray}
According to the condition, for any $a^{(i)}$ and $s^{(i)}$, $1-\delta \le \frac{\widehat{\mu}(a^{(i)}_t|s^{(i)}_t)}{\mu(a^{(i)}_t|s^{(i)}_t)} \le 1+\delta$. Then 
\[ \frac{1}{1+\delta} \le \frac{\mu(a^{(i)}_t|s^{(i)}_t)}{\hat{\mu}(a^{(i)}_t|s^{(i)}_t)} \le \frac{1}{1-\delta}, \]
and:
\[ \left( \frac{1}{1+\delta} \right)^{H} \le \prod_{t=0}^{H-1} \frac{\mu(a^{(i)}_t|s^{(i)}_t)}{\hat{\mu}(a^{(i)}_t|s^{(i)}_t)} \le \left( \frac{1}{1-\delta} \right)^{H}, \]
So:
\[ \left| \prod_{t=0}^{H-1}\frac{\mu(a^{(i)}_t|s^{(i)}_t)}{\hat{\mu}(a^{(i)}_t|s^{(i)}_t)}- 1 \right| \le \max \left\{ \left( \frac{1}{1-\delta} \right)^{H}-1, 1-\left( \frac{1}{1+\delta} \right)^{H} \right\} \]
Plug this into Equation \ref{eqn:empiricalpibdiff}:
\begin{eqnarray}
\left| V_{IS}(\hat{\mu}) - V_{IS}(\mu) \right|  &\le& \max \left\{ \left( \frac{1}{1-\delta} \right)^{H}-1, 1-\left( \frac{1}{1+\delta} \right)^{H} \right\}\frac{1}{n}\sum_{i=1}^n \prod_{t=0}^{H-1}\frac{\pi(a^{(i)}_t|s^{(i)}_t)}{\mu(a^{(i)}_t|s^{(i)}_t)} R^{(i)}_{0:H-1} \nonumber \\
&=& \max \left\{ \left( \frac{1}{1-\delta} \right)^{H}-1, 1-\left( \frac{1}{1+\delta} \right)^{H} \right\}  V_{IS}(\mu)
\end{eqnarray}
Similarly, for the bias:
\begin{eqnarray}
\left| \mathbb{E} V_{IS}(\hat{\mu}) - v \right| & = & \left| \mathbb{E} V_{IS}(\hat{\mu}) - \mathbb{E}V_{IS}(\mu) \right| \\
&=& \left| \mathbb{E}\prod_{t=0}^{H-1}\frac{\pi(a^{(0)}_t|s^{(0)}_t)}{\hat{\mu}(a^{(0)}_t|s^{(0)}_t)} R^{(0)}_{0:H-1} - \mathbb{E} \prod_{t=0}^{H-1}\frac{\pi(a^{(0)}_t|s^{(0)}_t)}{\mu(a^{(0)}_t|s^{(0)}_t)} R^{(0)}_{0:H-1}\right| \\
&\le& \left| \prod_{t=0}^{H-1}\frac{\mu(a^{(0)}_t|s^{(0)}_t)}{\hat{\mu}(a^{(0)}_t|s^{(0)}_t)}- 1 \right| \left| \mathbb{E} \left[ \prod_{t=0}^{H-1}\frac{\pi(a^{(0)}_t|s^{(0)}_t)}{\mu(a^{(0)}_t|s^{(0)}_t)} R^{(0)}_{0:H-1} \right] \right| \\
&\le& \max \left\{ \left( \frac{1}{1-\delta} \right)^{H}-1, 1-\left( \frac{1}{1+\delta} \right)^{H} \right\} v
\end{eqnarray}
\end{proof}
We bound the error of IS estimates by the relative error of behavior policy estimates. Proposition 3 from Farajtabar et al.~\cite{farajtabar2018more} gave an expression for the bias when using an empirical estimate of behavior policy in IS. The result in Farajtabar et al.~\cite{farajtabar2018more} is similar to this proposition, but the authors did not explicitly bound the bias by the error of behavior policy. Note that this bound increases exponentially with the horizon $H$, which shows the accumulated error effect of behavior policy error. 

By using the tree MDP example in Jiang and Li \cite{jiang2015doubly} we can show that the order of magnitude $O(\exp(H))$ is tight: there exists an MDP and a policy estimator $\hat{\mu}$ with $\| (\hat{\mu} - \mu)/\mu \|_{\infty} = \delta$ such that the bias of $V_{IS}(\hat{\mu})$ is $O(\exp{(H)})$. Define a binary discrete tree MDP \cite{jiang2015doubly} as following: At each node in a binary tree, we can take two actions $a=0, a=1$, leading to the two next nodes with observations $o=0, o=1$. The state of a node is defined by the whole path to the root: $o_0a_0o_1a_i\dots o_h$. That means each node in the tree will have a unique state. The depth of the tree, as well as the horizon of the trajectories, is $H$. Only the leftmost leaf node (by always taking $a=0$) has non-zero reward $r=1$. Denote this state as the target state. The evaluation policy always takes action $a=0$ and the behavior policy $\mu$ is a uniform random policy. Let the estimated policy $\hat{\mu}$ differ from $\mu$ with $-\delta/2$ in all the state-action pairs on the path to target state. That means the action probability in $\hat{\mu}$ is $1/2 -\delta/2$ for all state-action pairs on the path to target state. The IS estimator with $\mu$ has expectation $1$ since it is unbiased. It is easy to verify that the IS estimator using $\hat{\mu}$ has expectation $\left(1-\delta \right)^{-H}$. Thus the bias is $O(\exp(H))$.

This result represents the worst-case upper bound on the bias of IS when using an estimated behaviour policy; the fact it is exponential in the trajectory length illustrates the problem when using IS without knowing the true behaviour policy. 
To support this result with an empirical example illustrating the challenge of using IS with an unknown behaviour policy for a real data distribution, consider Figure \ref{fig:bp_ope_err} which represents the error in OPPE (found using Per-Decision WIS) as we vary the accuracy of the behaviour policy estimation. Two different behaviour policies are considered.
The domain used in this example is a continuous 2D map ($s \in \mathbb{R}^2$) with a discrete action space, $\mathcal{A} = \{1, 2, 3, 4, 5\}$, with actions representing a movement of one unit in one of the four coordinate directions or staying in the current position. Gaussian noise of zero mean and specifiable variance is added onto state of the agent after each action, to provide environmental stochasticity. An agent starts in the top left corner of the domain and receives a positive reward within a given radius of the top right corner, and a negative reward within a given radius of the bottom left corner. 
The horizon is set to be 15. A k-Nearest Neighbours (kNN) model is used to estimate the behaviour policy distribution, given a set of training trajectories. The accuracy of the model is varied by changing the number of trajectories available and the number of neighbours used for behaviour policy estimation. 

This plot shows how IS suffers from very poor estimates with even slight errors in the estimated behaviour policy -- average absolute errors of as small as 0.06 can incur errors of up over 50\% in OPPE. This provides additional motivation for our approach -- we do not require the behaviour policy to be known for OPPE, avoiding the significant errors incurred by using incorrectly estimated behaviour policies. 

\begin{figure}[ht]
 \centering 
 \centerline{\includegraphics[width=5in]{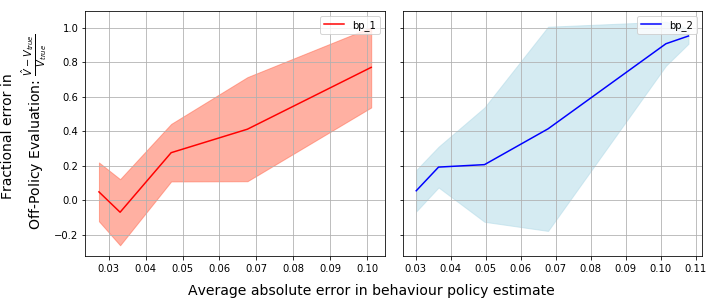} }
 \caption{Plots showing the mean and standard deviation of the fractional error in OPE, $\frac{\hat{V} - V}{V} $, as a function of the average absolute error in behaviour policy estimation, ${\frac{1}{n}\sum_{i=1}^{n}|\mu(a^{(i)}|s^{(i)}) - \hat{\mu}(a^{(i)}|s^{(i)})|}$, for two different behaviour policies. The quality of OPE with IS has a very significant dependence on the accuracy of the behaviour policy estimation.}
 \label{fig:bp_ope_err} 
\end{figure}

\section{Clarification of CATE/HTE and ITE}
\label{appendix:cate}
In the causal inference literature\cite{kunzel2017meta}, for an single unit $i$ with covariate (state) $x_i$, we observe $Y_i(1)$ if we give the unit treatment and $Y_i(0)$ if not. The Individual Treatment Effect (ITE) is defined as:
\begin{equation}
D_i = Y_i(1) - Y_i(0),
\end{equation}
for this particular set of observations $Y_i, x_i$. However $Y_i(1)$ and $Y_i(0)$ cannot be observed at the same time, which makes ITE unidentifiable without strong additional assumptions. Thus the conditional average treatment effect (CATE), also known as heterogeneous treatment effect (HTE) is defined as:
\begin{equation}
\tau(x) = \mathbb{E}[Y_i(1) - Y_i(0)|x]
\end{equation}
which is a function of $x$ and is identifiable. Shalit et al.~\cite{shalit2016estimating} defined ITE as $\tau(x)$, which is actually named as CATE or HTE in most causal reasoning literature. So we use the name CATE/HTE to refer to this quantity and it is inconsistent with Shalit et al.~'s work. We clarify it here so that it does confuse the reader.

\section{Proofs of Section \ref{sec:theory}}
\subsection{Proofs of Theorem \ref{thm:MSEVpi} and Corollary \ref{cor:generalization}}
Before we prove Lemma \ref{lem:MSEVpirecursive} and Theorem \ref{thm:MSEVpi}, we need some useful lemmas and assumptions. We restate a well-known variant of Simulation Lemma \cite{kearns2002near} in finite horizon case here:
\begin{lemma} (Simulation Lemma with finite horizon case)
\label{lem:simulation}
Define that $V_{\widehat{M},0}^\pi(s) = V_{M,0}^\pi(s) = 0$. For any approximate MDP model $\widehat{M}$, any policy $\pi$, and $t = 0, \dots, H-1$:
\begin{multline}
\label{eqn:simulationlemmarecursive}
V_{\widehat{M},H-t}^\pi(s_t) - V_{M,H-t}^\pi(s_t) = 
\mathbb{E}_{a_t \sim \pi}\left[ \widehat{r}(s_t,a_t) - \bar{r}(s_t,a_t) + \int_{\mathcal{S}} \left( \widehat{T}(s'|s_t,a_t) -T(s'|s_t,a_t) \right) \right. \\
\left. V_{\widehat{M},H-t-1}^{\pi}(s')  ds'  + \int_{\mathcal{S}}T(s'|s_t,a_t) \left(V_{\widehat{M},H-t-1}^\pi(s') - V_{M,H-t-1}^\pi(s') \right) ds' \right]
\end{multline}
Then:
\begin{multline}
V_{\widehat{M},H}^\pi(s) - V_{M,H}^\pi(s) = \mathbb{E}_{\pi,M} \sum_{t=0}^{H-1}\left[ \widehat{r}(s_t,a_t) - \bar{r}(s_t,a_t)\right.  \\
 \left. + \int_{\mathcal{S}} \left( \widehat{T}(s'|s_t,a_t) -T(s'|s_t,a_t) \right)V_{\widehat{M},H-t-1}^{\pi}(s') ds' |s_0 = s\right]
\end{multline}
\end{lemma}

\begin{lemma}
\label{lem:changeofvar_seq}
Let $J_{\psi}(z)$ be the absolute of the determinant of the Jacobian of $\psi(z)$.
 Then for any $z_t = \phi(s_t)$ and any sequence of actions $a_{0:t} = a_0, \dots, a_t$:
\begin{gather*}
p_{M,\mu}^{\phi}(z_t|a_{0:t})=J_{\psi}(z_t) p_{M,\mu}(\psi(z_t)|a_{0:t})
\end{gather*}
\begin{proof}
By the change of variable formula in a probability density function, we have:
$$p_{M,\mu}^{\phi}(z_t|a_{0:t}) = \frac{p_{M,\mu}^{\phi}(z_t, a_{0:t})}{p_{M,\mu}^{\phi}(a_{0:t})} 
= \frac{p_{M,\mu}(\psi(z_t),a_{0:t})J_{\psi}(z_t) }{p_{M,\mu}(a_{0:t})} = J_{\psi}(z_t) p_{M,\mu}(\psi(z_t)|a_{0:t}) $$
\end{proof}
\end{lemma}
\begin{lemma}
\label{lem:IPMbound}
Let $p$ and $q$ be two distributions over the state space with the form of $p_{M, \mu}(s_t|a_{0:t})$(the action sequence $a_{0:t}$ might be different for p and q), and $p^{\phi}$ and $q^{\phi}$ be the corresponding distributions over the representation space. For any real valued function over the state space $f$, if there exists a constant $B_{\phi}>0$ and a function class $G$ such that:
$ \frac{1}{B_{\phi}}f(\psi(z)) \in G $
then we have that
$$ \int_{\mathcal{S}} f(s)p(s)ds - \int_{\mathcal{S}} f(s)q(s)ds \le B_{\phi} \text{IPM}_{G}(p^{\phi},q^{\phi}) $$
\begin{proof}
\begin{eqnarray}
 \int_{\mathcal{S}} f(s)p(s)ds - \int_{\mathcal{S}} f(s)q(s)ds &=& \int_{\mathcal{S}} f(s)(p(s) - q(s))ds \\
 &=& \int_{\mathcal{Z}} f(\psi(z))(p(\psi(z)) - q(\psi(z)))J_{\psi}(z) dz \\
 &=& \int_{\mathcal{Z}} f(\psi(z))(p^{\phi}(z) - q^{\phi}(z)) dz \\
 &=& B_{\phi}\int_{\mathcal{Z}} \frac{1}{B_{\phi}} f(\psi(z))(p^{\phi}(z) - q^{\phi}(z))dz \\
 &\le&  B_{\phi} \left|  \int_{\mathcal{Z}} \frac{1}{B_{\phi}} f(\psi(z))(p^{\phi}(z) - q^{\phi}(z))dz \right|  \\
 &\le& B_{\phi}\sup_{g \in G}\left|  \int_{\mathcal{Z}} g(z)(p^{\phi}(z) - q^{\phi}(z))dz \right| \\
 &=& B_{\phi}\text{IPM}_{G}(p^{\phi},q^{\phi})
\end{eqnarray}
\end{proof}
\end{lemma}

The following lemma recursively bounds $\epsilon_{V}(\widehat{M},H-t)$ by $\epsilon_{V}(\widehat{M},H-t-1)$, whose result allows us to bound $\text{MSE}_{\pi} = \epsilon_{V}(\widehat{M},H)$.

The main idea to prove this is using Equation \ref{eqn:simulationlemmarecursive} from simulation lemma to decompose the loss of value functions into a one step reward loss, a transition loss and a next step value loss, with respect to the on-policy distribution. We can treat this as a contextual bandit problem, with the right side of Equation \ref{eqn:simulationlemmarecursive} as the loss function. For the distribution mismatch term, we follow the method in Shalit et al.'s work \cite{shalit2016estimating} about binary action bandits to bound the distribution mismatch by a representation distance penalty term. By converting the next step value error in the right side of Equation \ref{eqn:simulationlemmarecursive} into $\epsilon_V(\widehat{M},H-t-1)$,  we can repeat this process recursively to bound the value error for H steps.

\begin{lemma}
\label{lem:MSEVpirecursive}
For any MDP $M$, approximate MDP model $\widehat{M}$, behavior policy $\mu$ and deterministic evaluation policy $\pi$, let $B_{\phi,t}$ and $G_{t}$ be a scalar and a function family such that: 
\begin{multline}
\frac{1}{B_{\phi,t}}\left[ \bar{\ell}_{r}(\psi(z_t),\pi(\psi(z_t)),\widehat{M}) + \bar{\ell}_{T}(\psi(z),\pi(\psi(z_t)),\widehat{M}) \right. \\
+ \left. \frac{1}{2(H-t-1)}\int_{\mathcal{Z}}h_{T}(z'|\psi(z_t),\pi(\psi(z_t)))\bar{\ell}_{V}(\psi(z'),\widehat{M},H-t-1) dz' \right] 
\in G_{t}
\end{multline}
Then for any $t \le H-1$:
\begin{eqnarray*}
\epsilon_{V}(\widehat{M},H-t) \le (H-t) \left[ 2\int_{\mathcal{S}} \left[ \bar{\ell}_{r}(s_t,\pi(s_t),\widehat{M}) + \bar{\ell}_{T}(s_t,\pi(s_t),\widehat{M}) \right]p_{M,\mu}(s_t|a_{0:t} = \pi)ds_t \right. \\
\left. + \frac{\epsilon_{V}(\widehat{M},H-t-1)}{H-t-1}+ 2B_{\phi,t}\text{IPM}_{G_{t}}\left(p_{M,\mu}^{\phi}(z_t|a_{0:t}=\pi),p_{M,\mu}^{\phi}(z_t|a_t\neq \pi, a_{0:t-1}=\pi) \right) \right]
\end{eqnarray*}
\end{lemma}

\begin{proof}
Note the recursive form in Lemma \ref{lem:simulation}. We could treat RL as dealing with a contextual bandit problem at each step. Here we view the right side of recursive result in simulation lemma (restated here)
\begin{multline}
\label{eqn:one-step}
V_{\widehat{M},H-t}^\pi(s_t) - V_{M,H-t}^\pi(s_t) = 
\left[ \widehat{r}(s_t,\pi(s_t)) - \bar{r}(s_t,\pi(s_t)) \right. \\
 + \int_{\mathcal{S}} \left( \widehat{T}(s'|s_t,\pi(s_t)) -T(s'|s_t,\pi(s_t)) \right)V_{\widehat{M},H-t-1}^{\pi}(s')  ds'\\
\left. + \int_{\mathcal{S}}T(s'|s_t,\pi(s_t)) \left(V_{\widehat{M},H-t-1}^\pi(s_t) - V_{M,H-t-1}^\pi(s_t) \right) ds' \right]
\end{multline}
as a kind of square loss for a one-step prediction problem, and we bound the whole loss by recursively bounding these one-step losses. The key here is to find the recursive form of this square loss.

Recall that the definition of $\bar{\ell}_r, \bar{\ell}_T$
We will apply Cauchy-Schwarz inequality to bound Equation \ref{eqn:one-step}. Note that if $X_n = X_{n-1}+a_n+b_n$. Then $X_n^2 = (\frac{X_{n-1}}{\sqrt{n-1}}\sqrt{n-1}+\sqrt{2}a_n\frac{1}{\sqrt{2}}+\sqrt{2}b_n\frac{1}{\sqrt{2}})^2 \le (\frac{X_{n-1}^2}{n-1}+2a_n^2+2b_n^2)n $. By applying this to Equation \ref{eqn:one-step}, we have that
\begin{eqnarray}
\epsilon_{V}(\widehat{M},H-t) &=& \int_{\mathcal{S}} \bar{\ell}_{V}(s_t,H-t) p_{M,\mu}(s_t|a_{0:t-1} = \pi) ds_t \\
&\le& (H-t) \int_{\mathcal{S}} \left[ 2\bar{\ell}_{r}(s_t,\pi(s_t),\widehat{M}) + 2\bar{\ell}_{T}(s_t,\pi(s_t),\widehat{M}) \right. \\
&& \left. + \frac{1}{H-t-1}\int_{\mathcal{S}}T(s'|s_t,\pi(s_t))\bar{\ell}_{V}(s,\widehat{M},H-t-1) ds' \right] p_{M,\mu}(s_t|a_{0:t-1} = \pi) ds_t \nonumber
\end{eqnarray}
Note that:
\begin{eqnarray}
p_{M,\mu}(s_t|a_{0:t-1} = \pi) &=& p_{M,\mu}(s_t,a_t=\pi|a_{0:t-1} = \pi) + p_{M,\mu}(s_t,a_t\neq \pi|a_{0:t-1} = \pi) \\
&=& p_{M,\mu}(s_t|a_{0:t} = \pi)p(a_t=\pi|a_{0:t-1} = \pi) \\ && + p_{M,\mu}(s_t|a_t\neq \pi,a_{0:t-1} = \pi)p(a_t\neq \pi|a_{0:t-1} = \pi)
\end{eqnarray}
Let $c_t = p(a_t=\pi|a_{0:t-1} = \pi)$ then $1-c_t = p(a_t\neq \pi|a_{0:t-1} = \pi)$. 
Then \begin{eqnarray}
&& \epsilon_{V}(\widehat{M},H-t) \nonumber \\
&\le& c_t(H-t) \int_{\mathcal{S}} \left[ 2\bar{\ell}_{r}(s_t,\pi(s_t),\widehat{M}) + 2\bar{\ell}_{T}(s_t,\pi(s_t),\widehat{M}) \right. \nonumber \\
&& \left. + \frac{1}{H-t-1}\int_{\mathcal{S}}T(s'|s_t,\pi(s_t))\bar{\ell}_{V}(s,\widehat{M},H-t-1) ds' \right] p_{M,\mu}(s_t|a_{0:t} = \pi) ds_t \nonumber \\
&& + (1-c_t)(H-t) \int_{\mathcal{S}} \left[ 2\bar{\ell}_{r}(s_t,\pi(s_t),\widehat{M}) + 2\bar{\ell}_{T}(s_t,\pi(s_t),\widehat{M})+ \frac{1}{H-t-1} \right. \nonumber \\
&& \left. \int_{\mathcal{S}}T(s'|s_t,\pi(s_t))\bar{\ell}_{V}(s,\widehat{M},H-t-1) ds' \right] p_{M,\mu}(s_t|a_t \neq \pi, a_{0:t-1} = \pi) ds_t 
\label{eqn:before_IPMbound}
\end{eqnarray}

Let 
\[
f(s_t) = 2\bar{\ell}_{r}(s_t,\pi(s_t),\widehat{M}) + 2\bar{\ell}_{T}(s_t,\pi(s_t),\widehat{M}) + \frac{1}{H-t-1}\int_{\mathcal{S}}T(s'|s_t,\pi(s_t))\bar{\ell}_{V}(s,\widehat{M},H-t-1) ds' \]
We are going use Lemma \ref{lem:IPMbound} to bound the difference between $\int_{S} f(s_t)p_{M,\mu}(s_t|a_{0:t} = \pi) ds_t$ and $\int_{S} f(s_t)p_{M,\mu}(s_t|a_t \neq \pi, a_{0:t-1} = \pi) ds_t$. Let $G_{t}$ and $B_{\phi,t}>0$ be a function class and a constant satisfying that:
\begin{multline}
\label{eqn:bgcond1}
\frac{1}{B_{\phi,t}}\left[ \bar{\ell}_{r}(\psi(z_t),\pi(\psi(z_t)),\widehat{M}) + \bar{\ell}_{T}(\psi(z),\pi(\psi(z_t)),\widehat{M}) \right. \\
+ \left. \frac{1}{2(H-t-1)}\int_{\mathcal{Z}}h_{T}(z'|\psi(z_t),\pi(\psi(z_t)))\bar{\ell}_{V}(\psi(z'),\widehat{M},H-t-1) dz' \right] 
\in G_{t}
\end{multline}
Then following Lemma \ref{lem:IPMbound} we have that:
\begin{multline}
\int_{S} f(s_t)p_{M,\mu}(s_t|a_t \neq \pi, a_{0:t-1} = \pi) ds_t  \\
\le \int_{S} f(s_t)p_{M,\mu}(s_t|a_{0:t} = \pi) ds_t + 2B_{\phi,t}\text{IPM}_{G_{t}}\left(p_{M,\mu}^{\phi}(z_t|a_{0:t}=\pi),p_{M,\mu}^{\phi}(z_t|a_t\neq \pi, a_{0:t-1}=\pi) \right) \\
 = \int_{S} f(s_t)p_{M,\mu}(s_t|a_{0:t} = \pi) ds_t + 2B_{\phi,t}\text{IPM}_{G_{t}}\left(p_{M,\mu}^{\phi,F}(z_t),p_{M,\mu}^{\phi,CF}(z_t) \right) 
\end{multline}
Substituting this into Equation \ref{eqn:before_IPMbound} we have that
\begin{eqnarray}
&& \epsilon_{V}(\widehat{M},H-t) \nonumber \\
&\le& (H-t)\int_{S} f(s_t)p_{M,\mu}(s_t|a_{0:t} = \pi) ds_t + 2(1-c_t)(H-t)B_{\phi,t}\text{IPM}_{G_{t}}\left(p_{M,\mu}^{\phi,F}(z_t),p_{M,\mu}^{\phi,CF}(z_t) \right) \\
&\le& (H-t) \int_{\mathcal{S}} \left[ 2\bar{\ell}_{r}(s_t,\pi(s_t),\widehat{M}) + 2\bar{\ell}_{T}(s_t,\pi(s_t),\widehat{M}) + \frac{\int_{\mathcal{S}}\widehat{T}(s'|s_t,\pi(s_t))\bar{\ell}_{V}(s,\widehat{M},H-t-1) ds'}{H-t-1} \right] \nonumber \\
&& p_{M,\mu}(s_t|a_{0:t} = \pi) ds_t + 2(1-c_t)(H-t)B_{\phi,t}\text{IPM}_{G_{t}}\left(p_{M,\mu}^{\phi,F}(z_t),p_{M,\mu}^{\phi,CF}(z_t) \right)  \\
&\le& 2(H-t)\int_{\mathcal{S}} \left[ \bar{\ell}_{r}(s_t,\pi(s_t),\widehat{M}) + \bar{\ell}_{T}(s_t,\pi(s_t),\widehat{M}) \right]p_{M,\mu}(s_t|a_{0:t} = \pi)ds_t \nonumber \\ 
&& + \frac{H-t}{H-t-1}\epsilon_{V}(\widehat{M},H-t-1)  + 2(H-t)B_{\phi,t}\text{IPM}_{G_{t}}\left(p_{M,\mu}^{\phi,F}(z_t),p_{M,\mu}^{\phi,CF}(z_t) \right)
\end{eqnarray}
Thus we finish the proof.
\end{proof}
Iteratively applying this result for $t=0,1,\dots, H$ we will have Theorem \ref{thm:MSEVpi}. Note that $\frac{1}{p_{M,\mu}(a_{0:t}=\pi)} p_{M,\mu}(s_t,a_{0:t}=\pi) = p_{M,\mu}(s_t|a_{0:t}=\pi)$.

\begin{theorem}
(Restated)
For any MDP $M$, approximate MDP model $\widehat{M}$, behavior policy $\mu$ and deterministic evaluation policy $\pi$, let $B_{\phi,t}$ and $G_t$ be a real number and function family that satisfies the condition in Lemma \ref{lem:MSEVpirecursive}. Then:

\begin{multline*}
\mathbb{E}_{s_0}\left[ V^\pi_{\widehat{M}}(s_0) - V^\pi_{M}(s_0) \right]^2 \le 2H \sum_{t=0}^{H-1} \left[ B_{\phi,t} \text{IPM}_{G_{t}}\left(p_{M,\mu}^{\phi, F}(z_t),p_{M,\mu}^{\phi, CF}(z_t)\right) \right. \\
\left. + \int_{\mathcal{S}}\frac{1}{p_{M,\mu}(a_{0:t}=\pi)}\left( \bar{\ell}_{r}(s_t,\pi(s_t),\widehat{M})+\bar{\ell}_{T}(s_t,\pi(s_t),\widehat{M}) \right) p_{M,\mu}(s_t,a_{0:t}=\pi) d s_t  \right]
\end{multline*}
\end{theorem}

For $\text{MSE}_{\mu}$, we can apply the simulation lemma to bound it by the reward and transition losses since the data distribution matches the policy $\mu$. Then we combine it with the theorem above. Note that $\text{MSE}_{\pi} \le \text{MSE}_{\pi}+\text{MSE}_{\mu}$.

\begin{corollary}
\label{cor:generalization}
For any MDP $M$, approximate MDP model $\widehat{M}$, behavior policy $\mu$ and deterministic evaluation policy $\pi$, let $B_{\phi,t}$ and $G_t$ be a real number and function family that satisfies the condition in Lemma \ref{lem:MSEVpirecursive}. Let $u_{0:t} = p_{\mu,M}(a_{0:t}=\pi)$. Then:
\begin{multline*}
\text{MSE}_{\pi} \le \text{MSE}_{\pi}+\text{MSE}_{\mu} \\
\le 2H\sum_{t=0}^{H-1} \left[ \frac{1}{u_{0:t}}\int_{\mathcal{S}} \left( \bar{\ell}_{r}(s_t,\pi(s_t),\widehat{M}) + \bar{\ell}_{T}(s_t,\pi(s_t),\widehat{M}) \right) p_{M,\mu}(s_t,a_{0:t} = \pi)ds_t \right.+ \\
 \left. \int_{\mathcal{S}} \sum_{a_t \in \mathcal{A}} \left( \bar{\ell}_r(s_t,a_t,\widehat{M}) + \bar{\ell}_T(s_t,a_t,\widehat{M}) \right) p_{\mu,M}(s_t,a_t) ds_t \right. \\
 \left. +  B_{\phi,t}\text{IPM}_{G_{t}}\left(p_{M,\mu}^{\phi, F}(z_t),p_{M,\mu}^{\phi, CF}(z_t) \right) \right]
\end{multline*}
\end{corollary}

\begin{proof}
According to Lemma \ref{lem:simulation}, the mean square error of estimating the behavior policy value could be written as:
\begin{eqnarray}
&& \left[ \mathbb{E}_{s_0 } V^\mu_{\widehat{M}}(s_0) - \mathbb{E}_{s_0 } V^\mu_{M}(s_0) \right]^2 \nonumber \\
& = & \left[ \mathbb{E}_{\mu,M}  \left( \sum_{t=0}^{H-1}\widehat{r}(s_t,a_t) - \bar{r}(s_t,a_t) + \int_{\mathcal{S}} \left( \widehat{T}(s'|s_t,a_t) -T(s'|s_t,a_t) \right)V_{\widehat{M},H-t-1}^{\pi}(s') ds' \mid s_0 \right) \right]^2 \nonumber \\
& \le & \mathbb{E}_{\mu,M} \left[  \left( \sum_{t=0}^{H-1}\widehat{r}(s_t,a_t) - \bar{r}(s_t,a_t) + \int_{\mathcal{S}} \left( \widehat{T}(s'|s_t,a_t) -T(s'|s_t,a_t) \right)V_{\widehat{M},H-t-1}^{\pi}(s') ds' \right)^2 \right]  \\
& \le & \mathbb{E}_{\mu,M} \left[ 2H \sum_{t=0}^{H-1} \bar{\ell}_r(s_t,a_t,\widehat{M}) + \bar{\ell}_T(s_t,a_t,\widehat{M})  \right]  \\
& = & 2H \sum_{t=0}^{H-1} \int_{\mathcal{S}} \sum_{a_t} \left( \bar{\ell}_r(s_t,a_t,\widehat{M}) + \bar{\ell}_T(s_t,a_t,\widehat{M}) \right) p_{\mu,M}(s_t,a_t) ds_t 
\end{eqnarray}
The first step follows from Lemma \ref{lem:simulation}. The second step follows from Jensen's inequality, and the third step follows from Cauchy-Schwarz inequality. By combining the results above with Theorem \ref{thm:MSEVpi}, we have that: 
\begin{eqnarray}
&& \mathbb{E}_{s_0 } \left[ V^\pi_{\widehat{M}}(s_0) - V^\pi_{M}(s_0) \right]^2 \nonumber \\
&\le&  \mathbb{E}_{s_0 } \left[ V^\pi_{\widehat{M}}(s_0) - V^\pi_{M}(s_0) \right]^2 + \mathbb{E}_{s_0 } \left[ V^\mu_{\widehat{M}}(s_0) - V^\mu_{M}(s_0) \right]^2 \\
&\le&  2H\sum_{t=0}^{H-1} \left[ \int_{\mathcal{S}} \left( \bar{\ell}_{r}(s_t,\pi(s_t),\widehat{M}) + \bar{\ell}_{T}(s_t,\pi(s_t),\widehat{M}) \right) p_{M,\mu}(s_t|a_{0:t} = \pi)ds_t \right. \nonumber \\
&& + \int_{\mathcal{S}} \sum_{a_t \in \mathcal{A}} \left( \bar{\ell}_r(s_t,a_t,\widehat{M}) + \bar{\ell}_T(s_t,a_t,\widehat{M}) \right) p_{\mu,M}(s_t,a_t) ds_t \nonumber \\
&& \left.  + B_{\phi,t} \text{IPM}_{G_{t}}\left(p_{M,\mu}^{\phi,F}(z_t),p_{M,\mu}^{\phi,CF}(z_t) \right) \right] \\
&=& 2H\sum_{t=0}^{H-1} \left[ \frac{1}{u_{0:t}}\int_{\mathcal{S}} \left( \bar{\ell}_{r}(s_t,\pi(s_t),\widehat{M}) + \bar{\ell}_{T}(s_t,\pi(s_t),\widehat{M}) \right) p_{M,\mu}(s_t,a_{0:t} = \pi)ds_t \right. \nonumber \\
&& \left. + \int_{\mathcal{S}} \sum_{a_t \in \mathcal{A}} \left( \bar{\ell}_r(s_t,a_t,\widehat{M}) + \bar{\ell}_T(s_t,a_t,\widehat{M}) \right) p_{\mu,M}(s_t,a_t) ds_t \right] \nonumber\\
&& + 2H\sum_{t=0}^{H-1} B_{\phi,t}\text{IPM}_{G_{t}}\left(p_{M,\mu}^{\phi,F}(z_t),p_{M,\mu}^{\phi,CF}(z_t)  \right) 
\end{eqnarray}
\end{proof}

\subsection{Proof of Theorem \ref{thm:finite}}
We showed in Theorem \ref{thm:MSEVpi} that we can bound MSE by expected losses under the behavior policy distribution and an IPM term. In this section, we are going to further bound this by empirical losses and a generalization gap. We will firstly define some loss terms that are based on observations, instead of losses that are on expected values, $\bar{r}$ and $T(\cdot|s,a)$. Then we will introduce some lemmas that allow us to bound the generalization gap of weighted losses and IPM terms from previous works. Finally we will prove the finite sample MSE bound by putting these generalization gaps together.

\setcounter{definition}{3}
\begin{definition} 
(Restated)
Let $r_t$ and $s'_t$ be an observation of reward and next step given state action pair $s_t, a_t$. Define the loss function as:
\begin{align*}
\ell_{r}(s_t,a_t,r_t,\widehat{M}) &= \left( \widehat{r}(s_t,a_t) - r_t \right)^2 \\
\ell_{T}(s_t,a_t,s'_t,\widehat{M}) &= \left( \int_{\mathcal{S}} \left( \widehat{T}(s'|s_t,a_t) -\delta(s'-s'_t) \right)V_{\widehat{M},H-t-1}^{\pi}(s') ds' \right)^2 \\
&= \left( \int_{\mathcal{S}} \widehat{T}(s'|s_t,a_t) V_{\widehat{M},H-t-1}^{\pi}(s') ds' - V_{\widehat{M},H-t-1}^{\pi}(s'_t)  \right)^2
\end{align*}
where $\delta$ is the Dirac delta function.
\end{definition}
\setcounter{definition}{5}

Actually the difference between $\ell$ and $\bar{\ell}$ can be captured by the variance of the reward and transition dynamics, which only depend on the underlying dynamics. The following definition and lemmas show that.
\begin{definition}
Define the variance of $t$-th step reward and transition with respect to the state-action distribution $q(s_t,a_t)$ as:
\begin{gather*}
\sigma_{q,t} = \sigma_q(r)+\sigma_q(T) \quad \quad
\sigma_{q,t}(r) = \int_{\mathcal{S}}\sum_{a_t}\int_{\mathcal{R}} (r - \bar{r}(s_t,a_t))^2 p_{M}(r|s_t,a_t)q(s_t,a_t) drds_t \\
\sigma_{q,t}(T) = \int_{\mathcal{S}}\sum_{a_t}\int_{\mathcal{S}} \left( \int_{\mathcal{S}} T(s'|s_t,a_t) V_{\widehat{M},H-t-1}^{\pi}(s') ds' - V_{\widehat{M},H-t-1}^{\pi}(s'_t) \right)^2 p_{M}(s'_t|s_t,a_t)q(s_t,a_t) ds'_tds_t
\end{gather*}
\end{definition}
\begin{lemma} (Variance decomposition)
\label{lem:variance}
\begin{multline*}
\int_{\mathcal{S}} \sum_{a_t} \left( \bar{\ell}_r(s_t,a_t,\widehat{M}) + \bar{\ell}_T(s_t,a_t,\widehat{M}) \right) q(s_t,a_t) ds_t = \int_{\mathcal{S}} \sum_{a_t} \left( \int_{\mathcal{R}}\ell_r(s_t,a_t,r,\widehat{M})p(r|s_t,a_t)dr \right. \\
\left. + \int_{\mathcal{S}} \ell_T(s_t,a_t,s'_t,\widehat{M})p(s'_t|s_t,a_t,s'_t)ds'_t \right) q(s_t,a_t) ds_t -\sigma_{q,t} 
\end{multline*} 
\end{lemma}
\begin{proof}
Let's start with the $\ell_r$ and $\ell_T$ terms:
\begin{multline}
\ell_{r}(s_t,a_t,r_t,\widehat{M}) = \left( \widehat{r}(s_t,a_t) - r_t \right)^2 = \left( \widehat{r}(s_t,a_t) - \bar{r}(s_t,a_t) \right)^2 + \left( \bar{r}(s_t,a_t) - r_t \right)^2  \\
+ 2\left( \widehat{r}(s_t,a_t)- \bar{r}(s_t,a_t) \right)\left( \bar{r}(s_t,a_t) - r_t \right)
\end{multline}
Note that $\mathbb{E}\left[ \bar{r}(s_t,a_t) - r_t \right] = 0$ so the last term will be zero after we apply the integral. Then: 
\[
\int_{\mathcal{S}} \sum_{a_t} \int_{\mathcal{R}}\ell_r(s_t,a_t,r,\widehat{M})p(r|s_t,a_t) q(s_t,a_t) dr  ds_t = \int_{\mathcal{S}} \sum_{a_t} \bar{\ell}_r(s_t,a_t,\widehat{M}) q(s_t,a_t) ds_t + \sigma_{q,t}(r) 
\]
Similarly, for $\ell_T$ we have that:
\begin{eqnarray}
&& \ell_{T}(s_t,a_t,s'_t,\widehat{M}) \\
&=& \left( \int_{\mathcal{S}}  \widehat{T}(s'|s_t,a_t)V_{\widehat{M},H-t-1}^{\pi}(s') ds' -V_{\widehat{M},H-t-1}^{\pi}(s'_t) \right)^2 \\
&=& \left( \mathbb{E}_{s' \sim \widehat{T}}V_{\widehat{M},H-t-1}^{\pi}(s') -V_{\widehat{M},H-t-1}^{\pi}(s'_t) \right)^2 \\
&=& \left( \mathbb{E}_{ \widehat{T}}V_{\widehat{M},H-t-1}^{\pi}(s') -\mathbb{E}_{T}V_{\widehat{M},H-t-1}^{\pi}(s') \right)^2 + \left( \mathbb{E}_{T}V_{\widehat{M},H-t-1}^{\pi}(s') -V_{\widehat{M},H-t-1}^{\pi}(s'_t) \right)^2 \nonumber \\
&& + 2\left( \mathbb{E}_{\widehat{T}}V_{\widehat{M},H-t-1}^{\pi}(s') -\mathbb{E}_{T}V_{\widehat{M},H-t-1}^{\pi}(s') \right)\left( \mathbb{E}_{T}V_{\widehat{M},H-t-1}^{\pi}(s') -V_{\widehat{M},H-t-1}^{\pi}(s'_t) \right)
\end{eqnarray}
Note that 
\[ \mathbb{E}_{s'_t \sim T} \left[  \mathbb{E}_{s' \sim T}V_{\widehat{M},H-t-1}^{\pi}(s') -V_{\widehat{M},H-t-1}^{\pi}(s'_t) \right] =\mathbb{E}_{s' \sim T}V_{\widehat{M},H-t-1}^{\pi}(s')- \mathbb{E}_{s'_t \sim T} V_{\widehat{M},H-t-1}^{\pi}(s'_t)=0\] 
So the last term here will also be zero when we apply integral over $s'_t$ to $\ell_{T}(s_t,a_t,s'_t,\widehat{M})$ and we have that:
\begin{multline}
\int_{\mathcal{S}} \sum_{a_t} \int_{\mathcal{S}} \ell_T(s_t,a_t,s'_t,\widehat{M})p(s'_t|s_t,a_t,s'_t) q(s_t,a_t)ds'_t ds_t \\
= \int_{\mathcal{S}} \sum_{a_t} \bar{\ell}_T(s_t,a_t,\widehat{M})  q(s_t,a_t) ds_t +\sigma_{q,t}(T) 
\end{multline}
Thus we finished the proof by combining the $\ell_r$ part with $\ell_T$ part.
\end{proof}

Now we are going to bound the expected value of $\ell_r$ and $\ell_T$ terms by the empirical mean of it. We restate our definition about empirical risk and add the definition about corresponding generalization risk:
\setcounter{definition}{4}
\begin{definition}
(Restate)
\begin{eqnarray*}
R_{\mu}(\widehat{M}) &=& \sum_{t=0}^{H-1}\int_{\mathcal{S}} \sum_{a_t} \left( \int_{\mathcal{R}}\ell_r(s_t,a_t,r,\widehat{M})p(r|s_t,a_t)dr \right. \\
&&+ \left. \int_{\mathcal{S}} \ell_T(s_t,a_t,\widehat{M})p(s'_t|s_t,a_t)ds'_t \right) p_{\mu,M}(s_t,a_t) ds_t \nonumber \\
\widehat{R}_{\mu}(\widehat{M}) &=& \sum_{t=0}^{H-1}\int_{\mathcal{S}} \sum_{a_t} \left( \int_{\mathcal{R}}\ell_r(s_t,a_t,r,\widehat{M})p(r|s_t,a_t)dr \right. \\
&&+ \left. \int_{\mathcal{S}} \ell_T(s_t,a_t,\widehat{M})p(s'_t|s_t,a_t)ds'_t \right) \widehat{p}_{\mu,M}(s_t,a_t) ds_t \\
&=& \frac{1}{n} \sum_{i=1}^n \sum_{t=0}^{H-1} \ell_r(s_t^{(i)},a_t^{(i)},r^{(i)},\widehat{M}) + \ell_T(s_t^{(i)},a_t^{(i)},s'^{(i)}_t,\widehat{M}),
\end{eqnarray*}
where n is the number of trajectories and $s^{(i)}_t$ is the state of the $t$th step in the $i$th trajectory. Similarly we define $R_{\pi,u}$ and $\widehat{R}_{\pi,u}$:
\begin{eqnarray*}
R_{\pi,u}(\widehat{M}) &=& \sum_{t=0}^{H-1}\int_{\mathcal{S}} \sum_{a_t} \frac{\mathds{1}(a_{0:t}=\pi)}{u_{0:t}} \left( \int_{\mathcal{R}}\ell_r(s_t,a_t,r,\widehat{M})p(r|s_t,a_t)dr \right.\\
&& + \left. \int_{\mathcal{S}} \ell_T(s_t,a_t,\widehat{M})p(s'_t|s_t,a_t)ds'_t \right) p_{M,\mu}(s_t,a_t) ds_t \\
\widehat{R}_{\pi,u}(\widehat{M}) &=& \frac{1}{n} \sum_{i=1}^n \sum_{t=0}^{H-1} \frac{\mathds{1}(a^{(i)}_{0:t}=\pi)}{\widehat{u}_{0:t}}\left[ \ell_r(s_t^{(i)},a_t^{(i)},r^{(i)},\widehat{M}) + \ell_T(s_t^{(i)},a_t^{(i)},s'^{(i)}_t,\widehat{M}) \right],
\end{eqnarray*}
where $\widehat{u}_{0:t} = \sum_{i=1}^n \frac{\mathds{1}(a^{(i)}_{0:t}=\pi)}{n}$
\end{definition}
\setcounter{definition}{6}

We could bound $R_{\mu}$ and $R_{\pi, u}$ and pseudo-dimension which is a complexity term of the model class. We will use the learning bound about importance sampling and weighted importance sampling in Cortes et al. \cite{cortes2010learning} to bound $R_{\mu}$ and $R_{\pi, u}$. The following lemma is an immediate consequence of Corollary 2 and section 6 in Cortes et al. \cite{cortes2010learning}.
\begin{lemma}
\label{lem:ISbound}
For a hypothesis class $\mathcal{H}$ over input space $\mathcal{X}$, let $d$ be the pseudo-dimension of a real valued loss function class $\left\{ \ell_h(x), h \in \mathcal{H}, x \in \mathcal{X} \right\}$. $w(x)$ is a weighting function such that $\mathbb{E}_{p}[w(x)] = 1$. Let $\widehat{p}$ be the empirical distribution over n samples, and $\widehat{w}(x_i) = nw(x_i)/\sum_{i=1}^n w(x_i) $ is the normalized weights. For any $\ell$ in the loss function class. with probability $1-\delta$
\begin{eqnarray}
\label{eqn:weightedlearningbound}
\left| \mathbb{E}_p[w(x)\ell(x)] -  \mathbb{E}_{\widehat{p}}[\widehat{w}(x)\ell(x)] \right|  
&\le& \mathbb{V}_{p,\widehat{p}}[w, \ell]\frac{\mathcal{C}^{\mathcal{M}}_{n,\delta}}{n^{3/8}} + \ell_{\max}  \mathbb{V}_{p,\widehat{p}}[w,1]\frac{\mathcal{C}^{\mathcal{M}}_{n,\delta}}{n^{3/8}}
\end{eqnarray}
where $\mathcal{C}^{\mathcal{H}}_{n,\delta} = 2^{5/4}\left( d\log(2ne/d)+\log(4/\delta) \right)^{3/8}$, $\mathbb{V}_{p,\widehat{p}}[w, \ell] = \max \{\sqrt{\mathbb{E}[w(x)^2\ell(x)^2]}, \sqrt{\widehat{\mathbb{E}}[w(x)^2\ell(x)^2]} \}$, and $\ell_{\max} = \max_x |\ell(x)|$. 
\end{lemma}
\begin{proof}
We can decompose the gap into two parts by adding $\mathbb{E}_{\widehat{p}}[w(x)\ell(x)]$:
\begin{multline}
\left| \mathbb{E}_p[w(x)\ell(x)] -  \mathbb{E}_{\widehat{p}}[\widehat{w}(x)\ell(x)] \right| \\
\le \left| \mathbb{E}_p[w(x)\ell(x)] -  \mathbb{E}_{\widehat{p}}[w(x)\ell(x)] \right| + \left| \mathbb{E}_{\widehat{p}}[w(x)\ell(x)] -  \mathbb{E}_{\widehat{p}}[\widehat{w}(x)\ell(x)] \right|
\end{multline}
For the first part, we can bound it by Corollary 2 from Cortes et al. \cite{cortes2010learning}:
\begin{eqnarray}
\left| \mathbb{E}_p[w(x)\ell(x)] -  \mathbb{E}_{\widehat{p}}[w(x)\ell(x)] \right| &\le& \mathbb{V}_{p,\widehat{p}}[w, \ell]\frac{\mathcal{C}^{\mathcal{M}}_{n,\delta}}{n^{3/8}}
\end{eqnarray}
For the second part, according to section 6 from Cortes et al. \cite{cortes2010learning}, we have that
\footnote{Note that the definition of $\widehat{w}(x_i)$ in Cortes et al. \cite{cortes2010learning} is different with ours by a constant $n$. Here $\widehat{w}(x)$ follows our definition.}:
\begin{equation}
\left| \frac{1}{n} \left( \widehat{w}(x_i) - w(x_i) \right) \right| = \frac{w(x_i)}{W}\left| 1- \frac{W}{n} \right| \le \frac{w(x_i)}{W}\mathbb{V}_{p,\widehat{p}}[w,1]\frac{\mathcal{C}^{\mathcal{M}}_{n,\delta}}{n^{3/8}}
\end{equation}
where $W=\sum_i w(x_i)$. Then:
\begin{eqnarray}
\left| \mathbb{E}_{\widehat{p}}[(w(x)-\widehat{w}(x))\ell(x)]\right| &\le& \ell_{\max} \mathbb{E}_p[|w(x)-\widehat{w}(x)|] \\
&=& \ell_{\max} \left[\frac{1}{n}\sum_{i}\left|w(x_i)-\widehat{w}(x_i)\right|\right] \\
&\le& \ell_{\max}\mathbb{V}_{p,\widehat{p}}[w,1]\frac{\mathcal{C}^{\mathcal{M}}_{n,\delta}}{n^{3/8}}
\end{eqnarray}
Thus we finished the proof.
\end{proof}

We will apply this lemma to the risk at each time step $t$ separately. Since $\mathbb{E}_{M,\mu}[\frac{\mathds{1}(a^{(i)}_{0:t}=\pi)}{u_{0:t}}] = 1$ for each $t$, we can let $w(x) = \frac{\mathds{1}(a^{(i)}_{0:t}=\pi)}{u_{0:t}}$. In that case $\frac{\mathds{1}(a^{(i)}_{0:t}=\pi)}{\widehat{u}_{0:t}}$ is the normalized weights $\widehat{w}(x)$. We can also bound of $R_{\mu}$ from this as well, by setting the weight function to be one. In that case $w = \widehat{w} = 1$ and $\mathbb{V}_{p,\widehat{p}}[1, \ell] \le \ell_{\max} \le R_{\max}^2 + V_{\max,t}^2$ for the $t$th step loss function $\ell$.

For the IPM term, using norm-1 reproducing kernel Hilbert space (RKHS) function class for $G$ leads to IPM being the maximum mean discrepancy (MMD) distance. We can bound the gap between MMD distance and its empirical estimation using the following lemma in Sriperumbudur et al.'s work \cite{sriperumbudur2009integral}. There are many other choice such as of 1-Lipschitz functions, leading to Wasserstein distance, and $l_\infty$ norm unit ball, leading to total variation distance. There are similar results with those function class and distance measure, with worse bounds. We also use norm-1 RKHS functions and MMD metric in our experiment section.
\begin{lemma} (Theorem 11 from Sriperumbudur et al. \cite{sriperumbudur2009integral})
\label{lem:IPM}
Let $\mathcal{X}$ be a measurable space. Suppose k is measurable kernel such that $\sup_{x\in \mathcal{X}} k(x, x) \le C \le \infty$ and $\mathcal{H}$ the reproducing kernel Hilbert space induced by k. Let $\mathcal{F} = \{ f:  \| f \|_\mathcal{H}=1 \}$, and $\nu = \sup_{x\in \mathcal{X}, f \in \mathcal{F}} |f(x)| < \infty$. Then, with $\widehat{p}$, $\widehat{q}$ the empirical distributions of $p$, $q$ from $m_1$ and $m_2$ samples respectively, and with probability at least 1-$\delta$,
\begin{equation}
\left| \text{IPM}_{\mathcal{F}}(p,q) - \text{IPM}_{\mathcal{F}}(\widehat{p},\widehat{q}) \right| \le \sqrt{18\nu^2\ln(4/\delta)C}\left( \frac{1}{\sqrt{m_1}} + \frac{1}{\sqrt{m_2}} \right)
\end{equation}
\end{lemma}

\begin{theorem}
(Restated)
Suppose $\mathcal{M}_{\Phi}$ is a model class of model MDP models based on twice-differentiable, invertible state representation $\phi$'s: $\widehat{M}_{\phi} =\langle \widehat{r}(s,a), \widehat{T}(s',s,a) \rangle  = \langle h_{r}(\phi(s),a), h_{T}(\phi(s'),\phi(s),a) \rangle$. Given n H-step trajectories sampled from policy $\mu$, let the loss function for $(s_t,a_t)$ pair at $t^{\text{th}}$ step be $\ell_t(s_t,a_t, \widehat{M}_{\phi}) = \ell_{r}(s_t,a_t,r_t,\widehat{M}) + \ell_{T}(s_t,a_t,s'_t,\widehat{M})$. Let $d_t$ be the pseudo-dimension of function class $\{ \ell_t(s_t,a_t, \widehat{M}_{\phi}), \widehat{M}_{\phi} \in \mathcal{M}_{\Phi}\}$. Suppose $\mathcal{H}$ the reproducing kernel Hilbert space induced by k such that $\sup_{z\in \mathcal{Z}} k(z, z) \le C \le \infty$, and $\mathcal{F} = \{ f:  \| f \|_\mathcal{H}=1 \}$, and $\nu = \sup_{z\in \mathcal{X}, f \in \mathcal{F}} |f(z)| < \infty$. Assume there exist a constant $B_{\phi,t}$ such that $\frac{1}{B_{\phi,t}}\ell_t(\psi(z),\pi(\psi(z)), \widehat{M}_{\phi}) \in \mathcal{F}$. Then with probability $1-3\delta$, for any $\widehat{M} \in \mathcal{M}_{\Phi}$:
\begin{multline*}
\mathbb{E}_{s_0} \left[ V^\pi_{\widehat{M}}(s_0) - V^\pi_{M}(s_0) | \widehat{M} \right]^2 \le \text{MSE}_{\mu} + \text{MSE}_{\pi} \le 2H\widehat{R}_{\mu}(\widehat{M}) + 2H\widehat{R}_{\pi,u}(\widehat{M}) \\ 
+ 2H\sum_{t=0}^{H-1} B_{\phi,t} \left( \text{IPM}_{\mathcal{F}}\left(\widehat{p}_{M,\mu}^{\phi,F}(z_t),\widehat{p}_{M,\mu}^{\phi, CF}(z_t) \right) + \min \left\{ \mathcal{D}^{\mathcal{F}}_{\delta} \left( \frac{1}{\sqrt{m_{t,1}}}+\frac{1}{\sqrt{m_{t,2}}} \right), 2\nu \right\} \right) \\
+ 2H\sum_{t=0}^{H-1}\frac{\mathcal{C}^{\mathcal{M}}_{n,\delta, t}}{n^{3/8}} \left( \mathbb{V}_{p.\hat{p}}[\frac{\mathds{1}(a_{0:t}=\pi)}{\widehat{u}_{0:t}},\ell_t] + \mathbb{V}_{p.\hat{p}}[1,\ell_t] +\ell_{t,\max}\mathbb{V}_{p.\hat{p}}[\frac{\mathds{1}(a_{0:t}=\pi)}{u_{0:t}},1] \right)
\end{multline*}
$m_{t,1}$ and $m_{t,2}$ are the number of samples that used to estimate $\widehat{p}_{M,\mu}^{\phi,F}(z_t)$ and $\widehat{p}_{M,\mu}^{\phi,CF}(z_t)$ respectively. $\mathcal{D}^{\mathcal{F}}_{\delta} = \sqrt{18\nu^2\ln(4/\delta)C}$. $\mathcal{C}^{\mathcal{M}}_{n,\delta, t} = 2^{5/4}\left( d_t\log(2ne/d_t)+\log(4/\delta) \right)^{3/8}$. $\mathbb{V}_{p.\hat{p}}[w,\ell_t] = \max \{ \sqrt{\mathbb{E}_{p_{M,\mu}}[w(s_t,a_t)^2\ell_t(s_t,a_t)^2] }, \sqrt{\mathbb{E}_{\widehat{p}_{M,\mu}}[w(s_t,a_t)^2\ell_t(s_t,a_t)^2] } \}$. $\ell_{t,\max} = \max_{s_t,a_t} |\ell_t(s_t,a_t)| \le R_{\max}^2 + V_{\max,t}^2 $.
\end{theorem}
\begin{proof}
Applying Lemma \ref{lem:variance} to the result in Corollary \ref{cor:generalization} and plugging the definition of $R_{\mu}$ and $R_{\pi,u}$ in, we have that:
\begin{eqnarray}
&& \mathbb{E}_{s_0 } \left[ V^\pi_{\widehat{M}}(s_0)  - V^\pi_{M}(s_0)  \right]^2 \nonumber \\
&=& 2H\left( \sum_{t=0}^{H-1} B_{\phi,t}\text{IPM}_{\mathcal{F}}\left(p_{M,\mu}^{\phi,F}(z_t),p_{M,\mu}^{\phi,CF}(z_t)  \right) + R_{\mu}(\widehat{M}) + R_{\pi,u}(\widehat{M}) - \sigma \right) \\
&\le& 2H\left( \sum_{t=0}^{H-1} B_{\phi,t}\text{IPM}_{\mathcal{F}}\left(p_{M,\mu}^{\phi,F}(z_t),p_{M,\mu}^{\phi,CF}(z_t)  \right) + R_{\mu}(\widehat{M}) + R_{\pi,u}(\widehat{M}) \right)
\end{eqnarray}
where $\sigma$ is $\sum_{t=0}^{H-1}\sigma_{p_{M,\mu},t} +\sigma_{p_{M,\mu}(\cdot|a_{0:t}=\pi),t} \ge 0$. We will work term by term. First, we can use Lemma \ref{lem:IPM} for the IPM term:
\begin{equation}
\text{IPM}_{\mathcal{F}}\left(p_{M,\mu}^{\phi,F}(z_t),p_{M,\mu}^{\phi,CF}(z_t)  \right) \le \text{IPM}_{\mathcal{F}}\left(\widehat{p}_{M,\mu}^{\phi,F}(z_t),\widehat{p}_{M,\mu}^{\phi,CF}(z_t) \right) + \mathcal{D}^{\mathcal{F}}_{\delta} \left( \frac{1}{\sqrt{m_{t,1}}}+\frac{1}{\sqrt{m_{t,2}}} \right)
\end{equation}
At the same time, we know that for any two distribution $p,q$, $0\le \text{IPM}_{\mathcal{F}}(p,q) \le 2\nu$. So:
\begin{multline}
\label{eqn:IPMempiricalIPM}
\text{IPM}_{\mathcal{F}}\left(p_{M,\mu}^{\phi,F}(z_t),p_{M,\mu}^{\phi,CF}(z_t)  \right) \\
\le \text{IPM}_{\mathcal{F}}\left(\widehat{p}_{M,\mu}^{\phi,F}(z_t),\widehat{p}_{M,\mu}^{\phi,CF}(z_t) \right) + \min \left\{ \mathcal{D}^{\mathcal{F}}_{\delta} \left( \frac{1}{\sqrt{m_{t,1}}}+\frac{1}{\sqrt{m_{t,2}}} \right), 2\nu \right\}
\end{multline}
For $R_{\mu}$, if we plug $w(s,a)=\widehat{w}(s,a)=1$, $\ell = \ell_{t}(s_t,a_t,\widehat{M})$, and $p=p_{M,\mu}(s_t,a_t)$ into Lemma \ref{lem:ISbound}, we have that:
\begin{eqnarray}
R_{\mu} = \sum_{t=0}^{H-1} \mathbb{E}_{p}[\ell_t(s_t,a_t)] &\le& \sum_{t=0}^{H-1} \left( \mathbb{E}_{\widehat{p}}[\ell_t(s_t,a_t)]  +  \frac{\mathcal{C}^{\mathcal{M}}_{n,\delta,t}}{n^{3/8}}\mathbb{V}_{p,\widehat{p}}[1,\ell_t]\right) \\
&=& \widehat{R}_{\mu} + \sum_{t=0}^{H-1} \frac{\mathcal{C}^{\mathcal{M}}_{n,\delta,t}}{n^{3/8}}\mathbb{V}_{p,\widehat{p}}[1,\ell_t] 
\label{eqn:RmuempiricalRmu}
\end{eqnarray}
An analogous argument can be made for $R_{\pi,u}$ by letting $w(s_t,a_t) = \frac{\mathds{1}(a_{0:t}=\pi)}{u_{0:t}}$ which leads to that $\widehat{w}(s_t,a_t) = \frac{\mathds{1}(a_{0:t}=\pi)}{\widehat{u}_{0:t}}$:
\begin{eqnarray}
R_{\pi,u} &=& \sum_{t=0}^{H-1} \mathbb{E}_{p}[w(s_t,a_t)\ell_t(s_t,a_t)] \\
&\le& \sum_{t=0}^{H-1} \left( \mathbb{E}_{\widehat{p}}[\widehat{w}(s_t,a_t)\ell_t(s_t,a_t)]  +  \mathbb{V}_{p,\widehat{p}}[\widehat{w}, \ell_t]\frac{\mathcal{C}^{\mathcal{M}}_{n,\delta}}{n^{3/8}} + \ell_{t,\max}  \mathbb{V}_{p,\widehat{p}}[w,1]\frac{\mathcal{C}^{\mathcal{M}}_{n,\delta}}{n^{3/8}} \right) \\
&=& \widehat{R}_{\pi,u} + \sum_{t=0}^{H-1} \left( \mathbb{V}_{p,\widehat{p}}[\widehat{w}, \ell_t]\frac{\mathcal{C}^{\mathcal{M}}_{n,\delta}}{n^{3/8}} + \ell_{t,\max}  \mathbb{V}_{p,\widehat{p}}[w,1]\frac{\mathcal{C}^{\mathcal{M}}_{n,\delta}}{n^{3/8}} \right) 
\label{eqn:RpiuempiricalRpiu}
\end{eqnarray}
Thus we finish the proof by combining IPM terms, $R_{\mu}$ and $R_{\pi,u}$ together.
\end{proof}

\subsection{IS weights and marginal action probability ratio}
\label{appendix:theory_variance}
\begin{theorem}
\label{thm:variance}
For any evaluation policy $\pi$ and behavior policy $\mu$ satisfying that the support set of $\mu$ covers the support set of $\pi$, we have that the variance of importance sampling weights is no less than the variance of marginal action probability ratios:
\begin{eqnarray*}
    \mathrm{Var}_{\mu,M} \left[ \frac{\prod_{t=0}^{H-1}\pi(a_i|s_i) }{\prod_{t=0}^{H-1}\mu(a_i|s_i)} \right] 
    &\ge& \mathrm{Var}_{\mu,M} \left[ \frac{p_{\pi,M}(a_{0:H-1}) }{p_{\mu,M}(a_{0:H-1})} \right] \\
    &=& 
    \mathrm{Var}_{\mu,M} \left[ \frac{\int_{\mathcal{S}^H} \prod_{t=0}^{H-1}\pi(a_i|s_i)\prod_{t=0}^{H-1}T(s_i|s_{i-1},a_{i-1}) \mathrm{d} s_{0:t} }
    {\int_{\mathcal{S}^H} \prod_{t=0}^{H-1}\mu(a_i|s_i)\prod_{t=0}^{H-1}T(s_i|s_{i-1},a_{i-1}) \mathrm{d} s_{0:t}} \right] 
\end{eqnarray*}
where $T(s_0|s_{-1},a_{-1})$ is defined as the initial distribution $p_0(s_0)$.
\end{theorem}
\begin{proof}
\begin{gather}
    \mathrm{Var}_{\mu,M} \left[ \frac{\prod_{t=0}^{H-1}\pi(a_i|s_i) }{\prod_{t=0}^{H-1}\mu(a_i|s_i)} \right] = \mathbb{E}_{\mu,M} \left[ \left( \frac{\prod_{t=0}^{H-1}\pi(a_i|s_i) }{\prod_{t=0}^{H-1}\mu(a_i|s_i)} \right)^2 \right] - \left[ \mathbb{E}_{\mu,M} \left(  \frac{\prod_{t=0}^{H-1}\pi(a_i|s_i) }{\prod_{t=0}^{H-1}\mu(a_i|s_i)} \right) \right]^2 \label{eqn:variance_is} \\
    \mathrm{Var}_{\mu,M} \left[ \frac{p_{\pi,M}(a_{0:H-1}) }{p_{\mu,M}(a_{0:H-1}} \right] =  \mathbb{E}_{\mu,M} \left[ \left( \frac{p_{\pi,M}(a_{0:H-1}) }{p_{\mu,M}(a_{0:H-1})} \right)^2 \right] - \left[ \mathbb{E}_{\mu,M} \left( \frac{p_{\pi,M}(a_{0:H-1}) }{p_{\mu,M}(a_{0:H-1})} \right) \right]^2 \label{eqn:variance_margin}
\end{gather}
Among them, the expectation of marginal action probability ratio is equal to the expectation of IS weights: 
\begin{eqnarray}
    && \mathbb{E}_{\mu,M} \left( \frac{p_{\pi,M}(a_{0:H-1}) }{p_{\mu,M}(a_{0:H-1})} \right) \\
    &=& \int_{\mathcal{S}^H} \sum_{a_0, \dots, a_{H-1}} \prod_{t=0}^{H-1}\mu(a_i|s_i)\prod_{t=0}^{H-1}T(s_i|s_{i-1},a_{i-1}) \frac{p_{\pi,M}(a_{0:H-1}) }{p_{\mu,M}(a_{0:H-1})} \mathrm{d} s_{0:t}  \\
    &=& \sum_{a_0, \dots, a_{H-1}} \left( \int_{\mathcal{S}^H} \prod_{t=0}^{H-1}\mu(a_i|s_i)\prod_{t=0}^{H-1}T(s_i|s_{i-1},a_{i-1})\mathrm{d} s_{0:t} \right) \frac{p_{\pi,M}(a_{0:H-1}) }{p_{\mu,M}(a_{0:H-1})}   \\
    &=& \sum_{a_0, \dots, a_{H-1}}p_{\mu,M}(a_{0:H-1})\frac{p_{\pi,M}(a_{0:H-1}) }{p_{\mu,M}(a_{0:H-1})}  \\
    &=& \sum_{a_0, \dots, a_{H-1}}p_{\pi,M}(a_{0:H-1}) = 1\\
    && \mathbb{E}_{\mu,M} \left(  \frac{\prod_{t=0}^{H-1}\pi(a_i|s_i) }{\prod_{t=0}^{H-1}\mu(a_i|s_i)} \right) \\
    &=& \int_{\mathcal{S}^H} \sum_{a_0, \dots, a_{H-1}} \prod_{t=0}^{H-1}\mu(a_i|s_i)\prod_{t=0}^{H-1}T(s_i|s_{i-1},a_{i-1}) \left(  \frac{\prod_{t=0}^{H-1}\pi(a_i|s_i) }{\prod_{t=0}^{H-1}\mu(a_i|s_i)} \right) \mathrm{d} s_{0:t} \\
    &=& \sum_{a_0, \dots, a_{H-1}} \int_{\mathcal{S}^H}  \prod_{t=0}^{H-1}T(s_i|s_{i-1},a_{i-1}) \prod_{t=0}^{H-1}\pi(a_i|s_i)  \mathrm{d} s_{0:t} \\
    &=& \sum_{a_0, \dots, a_{H-1}}p_{\pi,M}(a_{0:H-1}) = 1
\end{eqnarray}
Thus the second term in Equation \ref{eqn:variance_is} and Equation \ref{eqn:variance_margin} are the same. Now we are going to prove that
\begin{equation}
    \mathbb{E}_{\mu,M} \left[ \left( \frac{\prod_{t=0}^{H-1}\pi(a_i|s_i) }{\prod_{t=0}^{H-1}\mu(a_i|s_i)} \right)^2 \right] \ge  \mathbb{E}_{\mu,M} \left[ \left( \frac{p_{\pi,M}(a_{0:H-1}) }{p_{\mu,M}(a_{0:H-1})} \right)^2 \right] 
\end{equation}
\begin{eqnarray}
&& \mathbb{E}_{\mu,M} \left[ \left( \frac{\prod_{t=0}^{H-1}\pi(a_i|s_i) }{\prod_{t=0}^{H-1}\mu(a_i|s_i)} \right)^2 \right] \\
&=& \int_{\mathcal{S}^H} \sum_{a_0, \dots, a_{H-1}} \prod_{t=0}^{H-1}\mu(a_i|s_i)\prod_{t=0}^{H-1}T(s_i|s_{i-1},a_{i-1}) \left( \frac{\prod_{t=0}^{H-1}\pi(a_i|s_i) }{\prod_{t=0}^{H-1}\mu(a_i|s_i)} \right)^2 \mathrm{d} s_{0:t} \\
&=& \sum_{a_0, \dots, a_{H-1}} \int_{\mathcal{S}^H}  \prod_{t=0}^{H-1}T(s_i|s_{i-1},a_{i-1})  \frac{\left( \prod_{t=0}^{H-1}\pi(a_i|s_i) \right)^2 }{\prod_{t=0}^{H-1}\mu(a_i|s_i)}  \mathrm{d} s_{0:t} \\
&=& \sum_{a_0, \dots, a_{H-1}} \int_{\mathcal{S}^H}  \frac{\left(\prod_{t=0}^{H-1}T(s_i|s_{i-1},a_{i-1})  \prod_{t=0}^{H-1}\pi(a_i|s_i) \right)^2 }{\prod_{t=0}^{H-1}T(s_i|s_{i-1},a_{i-1}) \prod_{t=0}^{H-1}\mu(a_i|s_i)}  \mathrm{d} s_{0:t}
\end{eqnarray}
\begin{eqnarray}
&&\mathbb{E}_{\mu,M} \left[ \left( \frac{p_{\pi,M}(a_{0:H-1}) }{p_{\mu,M}(a_{0:H-1})} \right)^2 \right] \\
&=& \int_{\mathcal{S}^H} \sum_{a_0, \dots, a_{H-1}} \prod_{t=0}^{H-1}\mu(a_i|s_i)\prod_{t=0}^{H-1}T(s_i|s_{i-1},a_{i-1}) \left( \frac{p_{\pi,M}(a_{0:H-1}) }{p_{\mu,M}(a_{0:H-1})} \right)^2 \mathrm{d} s_{0:t} \\
&=& \sum_{a_0, \dots, a_{H-1}} \left( \int_{\mathcal{S}^H}  \prod_{t=0}^{H-1}\mu(a_i|s_i)\prod_{t=0}^{H-1}T(s_i|s_{i-1},a_{i-1}) \mathrm{d} s_{0:t} \right) \left( \frac{p_{\pi,M}(a_{0:H-1}) }{p_{\mu,M}(a_{0:H-1})} \right)^2  \\
&=& \sum_{a_0, \dots, a_{H-1}}p_{\mu,M}(a_{0:H-1}) \left( \frac{p_{\pi,M}(a_{0:H-1}) }{p_{\mu,M}(a_{0:H-1})} \right)^2 \\
&=& \sum_{a_0, \dots, a_{H-1}} \frac{\left(p_{\pi,M}(a_{0:H-1}) \right)^2 }{p_{\mu,M}(a_{0:H-1})} \\
&=& \sum_{a_0, \dots, a_{H-1}} \frac{ \left(\int_{\mathcal{S}^H} \prod_{t=0}^{H-1}\pi(a_i|s_i)\prod_{t=0}^{H-1}T(s_i|s_{i-1},a_{i-1}) \mathrm{d} s_{0:t} \right)^2 }
{\int_{\mathcal{S}^H}\prod_{t=0}^{H-1}\mu(a_i|s_i)\prod_{t=0}^{H-1}T(s_i|s_{i-1},a_{i-1}) \mathrm{d} s_{0:t}} 
\end{eqnarray}
Now we only need to prove that for any $a_0, a_1, \dots a_{H-1}$:
\begin{eqnarray}
    \int_{\mathcal{S}^H}  \frac{\left(\prod_{t=0}^{H-1}T(s_i|s_{i-1},a_{i-1})  \prod_{t=0}^{H-1}\pi(a_i|s_i) \right)^2 }{\prod_{t=0}^{H-1}T(s_i|s_{i-1},a_{i-1}) \prod_{t=0}^{H-1}\mu(a_i|s_i)}  \mathrm{d} s_{0:t} \\
    \ge \frac{ \left(\int_{\mathcal{S}^H} \prod_{t=0}^{H-1}\pi(a_i|s_i)\prod_{t=0}^{H-1}T(s_i|s_{i-1},a_{i-1}) \mathrm{d} s_{0:t} \right)^2 }
{\int_{\mathcal{S}^H}\prod_{t=0}^{H-1}\mu(a_i|s_i)\prod_{t=0}^{H-1}T(s_i|s_{i-1},a_{i-1}) \mathrm{d} s_{0:t}} 
\end{eqnarray}
Multiplying both sides $\int_{\mathcal{S}^H}\prod_{t=0}^{H-1}\mu(a_i|s_i)\prod_{t=0}^{H-1}T(s_i|s_{i-1},a_{i-1}) \mathrm{d} s_{0:t}$ we have that:
\begin{eqnarray}
\int_{\mathcal{S}^H}\prod_{t=0}^{H-1}\mu(a_i|s_i)\prod_{t=0}^{H-1}T(s_i|s_{i-1},a_{i-1}) \mathrm{d} s_{0:t} \int_{\mathcal{S}^H}  \frac{\left(\prod_{t=0}^{H-1}T(s_i|s_{i-1},a_{i-1})  \prod_{t=0}^{H-1}\pi(a_i|s_i) \right)^2 }{\prod_{t=0}^{H-1}T(s_i|s_{i-1},a_{i-1}) \prod_{t=0}^{H-1}\mu(a_i|s_i)}  \mathrm{d} s_{0:t} \\
    \ge \left(\int_{\mathcal{S}^H} \prod_{t=0}^{H-1}\pi(a_i|s_i)\prod_{t=0}^{H-1}T(s_i|s_{i-1},a_{i-1}) \mathrm{d} s_{0:t} \right)^2 
\end{eqnarray}
This inequality holds by applying Cauchy-Schwarz inequality. Thus we finish the proof.
\end{proof}

\section{Proof of Section \ref{sec:algorithm}}
\begin{corollary}
\label{cor:consistency}
Let $\widehat{M}^*_{\phi^*} = \arg\min_{\widehat{M}_{\phi}}\mathcal{L}(\widehat{M}_{\phi}; \alpha_t)$ for a large enough $\alpha$ such that $\alpha_t > B_{\phi^*,t}$. Under the same definition and assumption in Theorem \ref{thm:finite}, we have that:
\begin{multline*}
\mathbb{E}_{s_0} \left[ V^\pi_{\widehat{M}^*_{\phi^*}}(s_0) - V^\pi_{M}(s_0)  \right]^2 \le O\left(\frac{1}{n^{3/8}}\right) + \sum_{t=1}^{H-1} O\left(\frac{1}{\sqrt{m_{t,1}}} + \frac{1}{\sqrt{m_{t,2}}}\right)  \\
+ 2H \min_{\widehat{M}_{\phi} \in \mathcal{M}} \left(R_{\mu}(\widehat{M}_{\phi}) + R_{\pi,u}(\widehat{M}_{\phi}) + \sum_{t=0}^{H-1}\alpha_t \text{IPM}_{\mathcal{F}}\left(p_{M,\mu}^{\phi,F}(z_t),p_{M,\mu}^{\phi,CF}(z_t) \right) \right) \nonumber
\end{multline*}
\end{corollary}
\begin{proof}
Let $M^*_{\phi^*_0}$ be the model that minimizes the expected risk and IPM term:
\[ M^*_{\phi^*_0} = \arg\min_{\widehat{M}_{\phi}} \left(R_{\mu}(\widehat{M}_{\phi}) + R_{\pi,u}(\widehat{M}_{\phi}) + \sum_{t=0}^{H-1}\alpha_t \text{IPM}_{\mathcal{F}}\left(p_{M,\mu}^{\phi,F}(z_t),p_{M,\mu}^{\phi,CF}(z_t) \right) \right)\]
From theorem \ref{thm:finite} we have that:
\begin{eqnarray}
&& \mathbb{E}_{s_0} \left[ V^\pi_{\widehat{M}^*_{\phi^*}}(s_0) - V^\pi_{M}(s_0)  \right]^2 \nonumber \\
&\le& O\left(\frac{1}{n^{3/8}}\right) + \sum_{t=1}^{H-1} O\left(\frac{1}{\sqrt{m_{t,1}}} + \frac{1}{\sqrt{m_{t,2}}}\right) + 2H\mathcal{L}(\widehat{M}^*_{\phi^*}; B_{\phi^*,t})  \\
&\le& O\left(\frac{1}{n^{3/8}}\right) + \sum_{t=1}^{H-1} O\left(\frac{1}{\sqrt{m_{t,1}}} + \frac{1}{\sqrt{m_{t,2}}}\right) + 2H\mathcal{L}(\widehat{M}^*_{\phi^*}; \alpha_t) \\
&\le& O\left(\frac{1}{n^{3/8}}\right) + \sum_{t=1}^{H-1} O\left(\frac{1}{\sqrt{m_{t,1}}} + \frac{1}{\sqrt{m_{t,2}}}\right) + 2H\mathcal{L}(M^*_{\phi^*_0}; \alpha_t)
\label{eqn:boundMSEbyoptimalL}
\end{eqnarray}
The first step follows from Theorem \ref{thm:finite}. The second step is from the fact that $\alpha_t > B_{\phi^*,t}$. The third step is from that $\widehat{M}^*_{\phi^*} = \arg\min_{\widehat{M}_{\phi}}\mathcal{L}(\widehat{M}_{\phi}; \alpha_t)$. Then we can bound the empirical loss term $\mathcal{L}(M^*_{\phi^*_0}; \alpha_t)$ by the expected value of it:
\begin{eqnarray}
\mathcal{L}(M^*_{\phi^*_0}; \alpha_t) &=& \widehat{R}_{\mu}(M^*_{\phi^*_0}) + \widehat{R}_{\pi,u}(M^*_{\phi^*_0}) + \sum_{t=0}^{H-1}\alpha_t \text{IPM}_{\mathcal{F}}\left(\widehat{p}_{M,\mu}^{\phi^*_0,F}(z_t),\widehat{p}_{M,\mu}^{\phi^*_0,CF}(z_t) \right) \nonumber \\
&& + \frac{\mathfrak{R}(\widehat{M}_{\phi})}{n^{3/8}} \\
&\le& R_{\mu}(M^*_{\phi^*_0}) + R_{\pi,u}(M^*_{\phi^*_0}) + \sum_{t=0}^{H-1}\alpha_t \text{IPM}_{\mathcal{F}}\left(p_{M,\mu}^{\phi^*_0,F}(z_t),p_{M,\mu}^{\phi^*_0,CF}(z_t) \right) \nonumber \\
&& + O\left(\frac{1}{n^{3/8}}\right) + O\left(\frac{1}{\sqrt{m_{t,1}}} + \frac{1}{\sqrt{m_{t,2}}}\right)
\end{eqnarray}
This follows from using Lemma \ref{lem:ISbound} and Lemma \ref{lem:IPM} similarly with Equation \ref{eqn:IPMempiricalIPM}, \ref{eqn:RmuempiricalRmu}, \ref{eqn:RpiuempiricalRpiu} but in different direction, together with the fact that $ \frac{\mathfrak{R}(\widehat{M}_{\phi})}{n^{3/8}} = O\left(\frac{1}{n^{3/8}}\right)$.

Put this into equation \ref{eqn:boundMSEbyoptimalL}, we have that
\begin{multline}
\mathbb{E}_{s_0} \left[ V^\pi_{\widehat{M}^*_{\phi^*}}(s_0) - V^\pi_{M}(s_0)  \right]^2 \le 2H \left( R_{\mu}(M^*_{\phi^*_0}) + R_{\pi,u}(M^*_{\phi^*_0}) \right. \\
+ \left. \sum_{t=0}^{H-1}\alpha_t \text{IPM}_{\mathcal{F}}\left(p_{M,\mu}^{\phi^*_0,F}(z_t),p_{M,\mu}^{\phi^*_0,CF}(z_t) \right) \right) + O\left(\frac{1}{n^{3/8}}\right) + \sum_{t=1}^{H-1} O\left(\frac{1}{\sqrt{m_{t,1}}} + \frac{1}{\sqrt{m_{t,2}}}\right)
\end{multline}
Thus we finished the proof.
\end{proof}
Under assumption about support set of $\mu$, $m_{t,1}, m_{t,2} \to \infty$ when $n \to \infty$. Then an immediate consequence from this corollary is that, if there exists an MDP and representation model in our model class that could achieve no generalization error,
\[
\min_{\widehat{M}_{\phi}} \left(R_{\mu}(\widehat{M}_{\phi}) + R_{\pi,u}(\widehat{M}_{\phi}) + \sum_{t=0}^{H-1}\alpha_t \text{IPM}_{\mathcal{F}}\left(p_{M,\mu}^{\phi,F}(z_t),p_{M,\mu}^{\phi,CF}(z_t) \right) \right) = 0,
\]
then $\lim_{n\to \infty}\mathbb{E}_{s_0} \left[ V^\pi_{\widehat{M}^*_{\phi^*}}(s_0) - V^\pi_{M}(s_0)  \right]^2 \to 0$ and estimator $V^\pi_{\widehat{M}^*_{\phi^*}}(s_0)$ is a consistent estimator for any $s_0$.

\section{Details of Experiment}

We will clarify the details of the Cart Pole and Mountain Car experiment and provide results from additional OPPE methods.

\textbf{Details of the domain} For Cart Pole domain, we follow the same settings as in the OpenAI Gym \cite{brockman2016openai} CartPole-v0 environment. The state consists of 4 features: position, speed, angle and angular speed. The agent can take two actions: move to the left or to the right. The trajectory will end either when the time step is larger than 200 or when the absolute value of position or angle is larger than the threshold. The goal in this domain is to control a cart as long as possible. We will receive the reward after each time step if the cart is under control, and the trajectory ends when the cart falls. 

We include two different variants in this domain: long horizon and short horizon. For long horizon, we learn a near-optimal Q function, and use the greedy policy as evaluation policy and $\epsilon-$greedy policy with $\epsilon=0.2$ as behavior policy. The average value, which is also the average length of trajectories, of the evaluation policy is $195$ and the average value of the behavior policy is $190$. For shorter horizon, we learn a weaker Q function and generate the policies in the same way, with the average value of $23.8$ and $24$ respectively. The reason that we learn a near optimal but not optimal policy for the long horizon is that the optimal policy can always hold the cart for 200 steps (max length), which makes it easy to estimate since there is no possibility of overestimating it.

For Mountain Car domain, we follow the same settings as in the OpenAI Gym \cite{brockman2016openai} MountainCar-v0 environment. The state consists of 4 features: position and velocity. The agent can take two actions: accelerate to the left or to the right. The trajectory will end either when the time step is larger than 200 or when the position exceeds the threshold. The goal in this domain is to control the car to reach the top of mountain as soon as possible. We will receive a negative reward after each time step. 

\textbf{Details of our model}
Our model has three parts: a representation module, a reward module, and a transition module. The representation module is a one layer feed-forward network that takes the state as input and outputs a 32-dimension representation. The reward module takes the representation as input and outputs $A=2$ predictions, corresponding to 2 different actions. The transition module is similar to the reward module, but it predicts the difference between state and next state which is a widely-used trick for transition dynamics modeling. Both the reward module and transition module are feed-forward networks with no hidden layer. We optimize the model using Adam. Since this domain has variable length of trajectories, we also learn the condition of terminal state. The only domain prior we assume is that we know the maximum length of a trajectory is 200.

We also need to explain the details of transition loss $\ell_T$. Since this domain is a deterministic domain, the loss function turns to be:
\begin{equation}
    \ell_T(s_t,a_t,s'_t,\widehat{M}) = \left( V^{\pi}_{\widehat{M},H-t-1}(s'_t) - V^{\pi}_{\widehat{M},H-t-1}(s') \right)^2,
\end{equation}
where $s'$ is the predicted next state prediction and $s'_{t}$ is the logged next state in dataset. Since repeatedly performing planning at training time is very computationally-intensive, it is difficult to get the function $V^{\pi}_{\widehat{M},H-t-1}(s)$. It is also challenging to compute the the derivative of this with respect to $s$. If we assume the resulting value is $L$-Lipschitz, then this loss can be bounded by $L(s'-s'_t)^2$. This is slightly different to the algorithmic part in the main body but it will still be an upper bound of $\ell_T$ in the main body. In this experiment we set $L=1$.

If we are in a discrete state space, the transition loss $\ell_T$ turns to be:
\begin{equation}
    \ell_T(s_t,a_t,s'_t,\widehat{M}) = \sum_{s' \in \mathcal{S}} V^{\pi}_{\widehat{M},H-t-1}(s') \left( \widehat{T}(s'|s,a) - \mathds{1}(s'_t-s')\right)^2
\end{equation}
We can use similar trick with double Q learning for DQN: doing value iteration to generate a target value vector $V^{\pi}_{\widehat{M},H-t-1}(s')$, and view this as constant vector when we compute derivative with $s'$. Then this loss becomes a weighted MSE loss. We can update the target value vector $V^{\pi}_{\widehat{M},H-t-1}(s')$ every several episodes.

\textbf{Methods} We compare several different methods: \textbf{1) \ourmethod} The proposed method. \textbf{2) AM} We compare our method $\ourmethod$ with a baseline approximate model, which uses the exactly same model class as our model, with the objective of minimizing the on-policy loss $R_{\mu}$. This is a straight-forward way to fit a regression model without any off-policy adjustment. \textbf{3) MRDR} we also compare with the more robust doubly robust (MRDR) method, which proposed a new way to train a Q function and use it in doubly robust. MRDR trains the Q function to minimize:
\begin{equation}
    \frac{1}{n} \sum_{i=1}^n \sum_{t=0}^{H-1} (w_{0:t}^{(i)})^2 \frac{1-\mu(a^{(i)}_t|s^{(i)}_t)}{\mu(a^{(i)}_t|s^{(i)}_t)} \left( \bar{R}_{t:H-1}^{(i)} - \widehat{Q}^{\pi}(s^{(i)}_t,a^{(i)}_t) \right),
\end{equation}
where $\bar{R}_{t:H-1}^{(i)}=\sum_{j=t+1}^{H-1} w_{t+1:j}^{(i)}r^{(i)}_j$ is the per-decision IS return from $t+1$ to $H-1$, and $w$'s are IS weights. \textbf{4) MRDR-WIS} Since this objective function can be very noisy and hard to fit when IS weights are high-variance, we also test another variant of MRDR by changing $\bar{R}_{t:H-1}^{(i)}$ to a weighted per-decision IS return from $t+1$ to $H-1$.

Within each one of the methods above, we test five different kinds of estimator. We have a pure MDP/Q model estimator, doubly robust (DR) using that MDP/Q model and weighted doubly robust (WDR) using that MDP/Q model. We evaluate a deterministic evaluation policy, which will result in most of the IS weights being zero; once an IS weight at one timestep is zero, then the product of all IS weights after that step will be zero. This setting is challenging for importance sampling and DR. We also test a very simple idea to avoid this problem -- we add a slight noise perturbation into the evaluation policy ($\epsilon$ = 0.01), and treat it as the true evaluation policy to generate IS weights for DR and WDR. The additional noise is small enough so that the error introduced by this is negligible compared with the MSEs of estimators. We call these variants of DR and WDR soft DR and soft WDR respectively.

We also compare with importance sampling (IS), weighted IS (WIS), soft IS, soft WIS, per-decision importance sampling (PDIS), weighted PDIS (WPDIS), soft PDIS, soft WPDIS. The soft methods are produced by changing the IS weights using the soft evaluation policy.

We report the results in Table \ref{table:cartfullresult} and \ref{table:mcfullresult}. Note that in the long horizon case, IS weights are all zero so WIS estimator is not defined. Though it is clear that for a single individual in continuous state space, IS and DR would not produce meaningful results due to the fact they only estimate from one trajectory, here we still include the IS and DR estimates for MSE for individual policy values. Not surprisingly we observe that those results are enormous which verifies that plain IS and DR are not reasonable estimators for individual value.

\begin{table}[ht!]
\caption{Root MSE for Cart Pole}
\label{table:cartfullresult}
\small
\centering

\begin{tabular}{lcccccccc}
\toprule
\multicolumn{1}{c}{} & 
\multicolumn{2}{c}{Long horizon} & 
\multicolumn{2}{c}{Short horizon}  
\\
\multicolumn{1}{c}{} & 
\multicolumn{1}{c}{MSE (mean)} & 
\multicolumn{1}{c}{MSE (individual)} & 
\multicolumn{1}{c}{MSE (mean)} & 
\multicolumn{1}{c}{MSE (individual)} 
\\
\midrule
\ourmethod	&\textbf{0.412}	&\textbf{1.033}	&0.078	&\textbf{0.481} \\
DR(\ourmethod)	&1.359	&40.820	&0.021	&0.789 \\
WDR(\ourmethod)	&0.619	&17.760	&0.026	&0.857 \\ 
Soft DR(\ourmethod)	&1.608	&53.390	&\textbf{0.020}	&0.686 \\
Soft WDR(\ourmethod)	&0.730	&24.95	&20.59	&634.7 \\
\midrule
AM	&0.754	&1.313	&0.125	&0.551 \\
DR(AM)	&1.786	&58.66	&0.024	&0.863 \\
WDR(AM)	&0.706	&19.73	&0.025	&0.929 \\
Soft DR(AM)	&1.613	&52.56	&\textbf{0.020}	&0.744 \\
Soft WDR(AM)	&0.848	&29.71	&20.28	&640.4 \\
AM ($\pi$) &41.80 &47.63 & 0.1233 &0.5974 \\
\midrule
MRDR's Q	&151.1	&151.9	&3.013	&3.823 \\
MRDR	&202.0	&7055	&0.258	&8.266 \\
WMRDR	&123.6	&1049	&2.343	&59.640 \\
Soft MRDR	&813.8	&2590	&0.211	&6.758 \\
Soft WMRDR	&92.00 &2669	&22.550	&601.4 \\
\midrule
MRDR-WIS's Q	&143.9	&145.1	&2.486	&3.440 \\
MRDR-WIS	&190.9	&6106	&0.248	&8.075 \\
WMRDR-WIS	&122.0	&1054	&2.599	&68.60 \\
Soft MRDR-WIS	&746.9	&23610	&0.199	&6.626 \\
Soft WMRDR-WIS	&108.3	&2992 &21.26	&570.6 \\
\midrule
IS	&194.500 &194.7	&2.860	&93.87\\ 
WIS	& -  &-  &0.505	&93.86 \\
Soft IS	&187.9	&1115	&2.179	&70.78 \\
Soft WIS	&8.144	&4698	&0.380	&70.55 \\
\midrule
PSIS	&477.5	&1526	&1.083	&36.67 \\
WPSIS	&125.9	&622.5	&1.819	&63.59 \\ 
Soft PSIS	&215.2	&6853	&0.903	&30.08 \\
Soft WPSIS	&4.225	&1983	&24	&678.2 \\
\bottomrule
\end{tabular}
\end{table}

\begin{table}[ht!]
\caption{Root MSE for Mountain Car}
\label{table:mcfullresult}
\small
\centering

\begin{tabular}{lccccc}
\toprule
\multicolumn{1}{c}{} & 
\multicolumn{1}{c}{MSE (mean)} & 
\multicolumn{1}{c}{MSE (individual)} 
\\

\midrule
\ourmethod &\textbf{12.31} & \textbf{31.38}  \\
DR(\ourmethod)&135.8 &	4352  \\
WDR(\ourmethod) &27.27 &	790.7 \\
Soft DR(\ourmethod) &59.9 &	1929 \\
Soft WDR(\ourmethod) &22.6 &	825.4 \\

\midrule
AM& 17.15&	36.36 \\
DR(AM)& 141.6&	4548\\
WDR(AM)& 24.89&	756.9\\
Soft DR(AM)& 66.45&	2129\\
Soft WDR(AM)& 23.79&	831.7\\
AM ($\pi$)& 72.61&	79.46\\
\midrule
MRDR's Q &135.4&	138.1\\
MRDR &172.7&	5427\\
WMRDR &78.34&	1400\\
Soft MRDR &4481&	125100\\
Soft WMRDR &5631&	104500\\
\midrule
MRDR-WIS's Q &140.5&	143.1\\
MRDR-WIS&212&	6975\\
WMRDR-WIS&110.1&	2101\\
Soft MRDR-WIS&5308&	139000\\
Soft WMRDR-WIS&8564&	167700\\
\midrule
IS &149.7&	152.2\\
WIS &nan&	nan\\
Soft IS &208.5&	3936\\
Soft WIS &301.3&	3862\\
\midrule
PSIS &108.6&	2334\\
WPSIS &99.79&	440.7\\
Soft PSIS &117.1&	3597\\
Soft WPSIS &45.8&	1924\\
\bottomrule
\end{tabular}
\end{table}

\textbf{Evaluation}
Thomas and Brunskill \cite{thomas2016data} discussed that it is not obvious how to use the trajectories to fairly compare DR, IS and AM estimators, in Appendix D.4 from \cite{thomas2016data}. There are three ways that are reasonable: the first way is that AM and DR estimators should be provided with additional trajectories that are not available to IS, which are used to learn the model. This can be viewed as the additional domain prior knowledge. This is the setting in MRDR's experiment \cite{farajtabar2018more}. The second way is that all methods should have the same amount of data. DR methods should split the data into two parts to learn the model and IS weights separately. That partition keeps the unbiasedness of DR, but reduces the size of available samples for model learning in DR. The third way is that all methods should have the same amount of data. The DR method reuses the data to learn the model and compute IS weights. This helps DR methods to achieve best empirical performance in Thomas and Brunskill \cite{thomas2016data}. There is not necessarily a "correct" answer to this question. We follow the third setting to make both DR and IS stronger baselines. 

We sample 1024 trajectories to generate off-policy estimators. For our method and AM method, we split the data into a training set (90$\%$) and a validation set ($10\%$) and use the validation set to tune the model structure and optimization settings. To compute the MSE of an individual value, we record the initial state of the 1024 trajectories and roll-out from true environment to get the true policy over those initial states as ground truth. We use the average policy value over these initial states as the ground truth for MSE of mean value. We repeat the whole process for $N=100$ runs and report the square root of averaged MSEs (for both individual and mean).

\textbf{Effect of parameter $\alpha$}

We study the effect of the hyper-paramter $\alpha$ in the IPM terms on the estimation results. We show the MSE of \ourmethod~trained using different $\alpha$. 

\begin{table}[ht!]
\caption{Root MSEs of \ourmethod~with different $\alpha$ for the cartpole domain}
\small
\centering

\begin{tabular}{lccccc}
\toprule
\multicolumn{1}{c}{Long horizon} & 
\multicolumn{1}{c}{$\alpha=0$} & 
\multicolumn{1}{c}{$\alpha=0.01$} & 
\multicolumn{1}{c}{$\alpha=0.1$} & 
\multicolumn{1}{c}{$\alpha=1$} &
\multicolumn{1}{c}{$\alpha=10$} 
\\
\midrule

Mean &0.554 &0.412 &0.406 &\textbf{0.389} &2.287 \\
Individual &1.178 &1.033 &\textbf{1.008} &1.023 &3.903 \\
\toprule
\multicolumn{1}{c}{Short horizon} & 
\multicolumn{1}{c}{$\alpha=0$} & 
\multicolumn{1}{c}{$\alpha=0.01$} & 
\multicolumn{1}{c}{$\alpha=0.1$} & 
\multicolumn{1}{c}{$\alpha=1$} &
\multicolumn{1}{c}{$\alpha=10$} 
\\
\midrule
Mean &0.114 &\textbf{0.078} &0.114 &0.357 &0.365 \\
Individual  &0.672 &\textbf{0.481} &0.702 &1.684 &1.545 \\
\bottomrule
\end{tabular}
\end{table}

\subsection{Further discussion}
An interesting issue is about the effect of the horizon. Although the "marginal" IS weights have less variance than IS weights, there is still a concern when the horizon is very long and the overlap of the behavior policy and the evaluation policy is small. That also has an effect on the IPM term: we would not have enough factual/counterfactual samples to estimate the IPMs, for large time steps $t$. In that case, the IPMs only effectively adjust the representation for the earlier of the trajectories. Both of the experimental domains actually encounter this case, and the experimental results show that \ourmethod~still outperforms other methods. In the Cart Pole domain it is clear that \ourmethod~can still benefit from IPM.

\end{document}